\documentclass[pdflatex,sn-mathphys-num]{sn-jnl}

\usepackage{graphicx}%
\usepackage{multirow}%
\usepackage{amsmath,amssymb,amsfonts}%
\usepackage{amsthm}%
\usepackage{mathrsfs}%
\usepackage[title]{appendix}%
\usepackage{xcolor}%
\usepackage{textcomp}%
\usepackage{manyfoot}%
\usepackage{booktabs}%
\usepackage{algorithm}%
\usepackage{algorithmicx}%
\usepackage{algpseudocode}%
\usepackage{listings}%
\usepackage{tikz}
\usetikzlibrary{arrows}
\usepackage {xfrac}
\usepackage {hyperref}

\theoremstyle{thmstyleone}%
\newtheorem{theorem}{Theorem}
\theoremstyle{thmstyletwo}%
\newtheorem{example}{Example}%
\newtheorem{remark}{Remark}%
\newtheorem{corollary}{Corollary}
\theoremstyle{thmstylethree}%
\newtheorem{definition}{Definition}%
\newtheorem{lemma}{Lemma}
\begin{document}

\title[Article Title]{Encoding argumentation frameworks to propositional logic systems}


\author*[1]{\fnm{Shuai} \sur{Tang}}\email{TangShuaiMath@outlook.com}

\author*[1]{\fnm{Jiachao} \sur{Wu}}\email{wujiachao@sdnu.edu.cn}

\author[1]{\fnm{Ning} \sur{Zhou}}\email{1679659758@qq.com}

\affil[1]{\orgdiv{School of Mathematics and Statistics}, \orgname{Shandong Normal University}, \orgaddress{\street{Daxue Road}, \city{Jinan}, \postcode{250358}, \state{Shandong}, \country{China}}}



\abstract{This paper generalizes the encoding of argumentation frameworks beyond the classical 2-valued propositional logic system ($PL_2$) to 3-valued propositional logic systems ($PL_3$s) and fuzzy propositional logic systems ($PL_{[0,1]}s$), employing two key encodings: normal encoding ($ec_1$) and regular encoding ($ec_2$). Specifically, via $ec_1$ and $ec_2$, we establish model relationships between Dung's classical semantics (stable and complete semantics) and the encoded semantics associated with Kleene's $PL_3$ and Łukasiewicz's $PL_3$. Through $ec_1$, we also explore connections between Gabbay's real equational semantics and the encoded semantics of $PL_{[0,1]}s$, including showing that Gabbay's $Eq_{\text{max}}^R$ and $Eq_{\text{inverse}}^R$ correspond to the fuzzy encoded semantics of $PL_{[0,1]}^G$ and $PL_{[0,1]}^P$ respectively. Additionally, we propose a new fuzzy encoded semantics ($Eq^L$) associated with Łukasiewicz's $PL_{[0,1]}$ and investigate interactions between complete semantics and fuzzy encoded semantics. This work strengthens the links between argumentation frameworks and propositional logic systems, providing a framework for constructing new argumentation semantics.}




\maketitle
\section{Introduction}
Since Dung came up with argumentation frameworks ($AF$s) \cite{ref 1}, this theory has been a hot topic in the fields of artificial intelligence \cite{ref 2} and law \cite{ref 3}. After Dung-style $AF$s, many variations of $AF$s have been proposed, such as $AF$s with sets of attacking arguments \cite{ref 4, ref 5}, the bipolar $AF$s \cite{ref 6, ref 7, ref 8} with two kinds of relations named the attack relation and the support relation, the higher-level $AF$s \cite{ref 9} with higher-level attack relations, and the bipolar $AF$s with higher-level relations \cite{ref 10}. Furthermore, numerical semantics have been developed in many forms, such as real equational $AF$s \cite{ref 11, ref 12}, fuzzy $AF$s \cite{ref 13, ref 14, ref 15, ref 16}, weighted $AF$s \cite{ref 17, ref 18, ref 19, ref 20, ref 21} and probabilistic $AF$s \cite{ref 22, ref 23, ref 24, ref 25, ref 26}.

A research branch connects $AF$s and logic. Logical $AF$s have been developed in many aspects \cite{ref 27, ref 28}. Besnard and Doutre \cite{ref 29, ref 30} propose the encoding approach which encodes $AF$s to the $2$-valued propositional logic system ($PL_2$). Abstract dialectical frameworks \cite{ref 31, ref 32} express acceptance conditions of each argument as a propositional formula. Cayrol et al. \cite{ref 33} propose extending arguments or attacks as accepted atoms and not accepted atoms and then encoding $AF$s as logical formulas. Arieli and Caminada \cite{ref 34, ref 35} suggest encoding an $AF$ as a quantified Boolean formula whose models coincide with the extensions of the $AF$. The seminal literature \cite{ref 1} first studies connections between $AF$ semantics and logic programming, and the approach of logic programming has been developed in \cite{ref 36, ref 37, ref 38, ref 39}. Moreover, some works \cite{ref 40, ref 41, ref 42} connect $AF$s with modal logic and some works \cite{ref 43, ref 44} translate $AF$s into intuitionistic logic.

Particularly, we emphasize works in \cite{ref 29, ref 12} and claim our motivations. Besnard and Doutre \cite{ref 29} give the approach of model checking which encodes an $AF$ as a propositional formula of the $PL_2$ such that models of the propositional formula are equivalent to extensions of the $AF$, and some other literature \cite{ref 27, ref 28, ref 45} reports this work. However, they only give the encoding approach related to $PL_2$. As a natural expansion, we encode $AF$s to general propositional logic systems ($PLS$s). Among various three-valued propositional logic systems ($PL_3$s), Kleene's $PL_3$ ($PL_3^K$) and Łukasiewicz's $PL_3$ ($PL_3^L$) stand out as particularly prominent due to their significant theoretical importance and widespread practical applications. Thus, we encode $AF$s to $PL_3^K$ and $PL_3^L$, and explore relationships between classical argumentation semantics and encoded semantics associated with $PL_3$s. Motivated by Gabbay's real equational semantics \cite{ref 12}, we encode $AF$s to fuzzy propositional logic systems ($PL_{[0,1]}$s) to derive $[0,1]$-valued semantics. This enables us to systematically analyze relationships between real equational semantics and our fuzzy encoded semantics.

In cases of $3$-valued semantics, we employ the normal encoding $ec_1$ and the regular encoding $ec_2$. It is shown that stable semantics is equivalent to the encoded semantics associated with $PL_3^K$ and $ec_1$, while complete semantics is equivalent to the encoded semantics associated with $PL_3^L$ and $ec_1$. The results are more complex for the regular encoding $ec_2$, but the model relationships are established by the binarization function and the ternarization function. In cases of $[0,1]$-valued semantics, we encode $AF$s to $PL_{[0,1]}$s by $ec_1$. Then we explore the relationship between Gabbay's real equational semantics and fuzzy encoded semantics. We also study relationships between complete semantics and fuzzy encoded semantics by the ternarization function. In addition, we propose a specific fuzzy encoded semantics via $ec_1$ and $PL_{[0,1]}^L$. 

This paper advances the model checking approach through encoding $AF$s to different $PLS$s. Specifically, we:
\begin{itemize}
	\item establish connections between the encoding methodology and Gabbay's real equational approach;
	\item extend the model checking approach associated with $PL_2$ as the general encoding methodology associated with general $PLS$s;
	\item encode $AF$s to $PL_3$s and study the model relationships between Dung's semantics and encoded semantics; 
	\item propose encoded semantics associated with $PL_{[0,1]}$s and investigate their relationships with Gabbay's equational semantics and complete semantics;
	\item provide logical foundations for $AF$ semantics by bridging $AF$s and $PLS$s;
	\item introduce a novel approach for deriving new argumentation semantics.
\end{itemize}

In Section 2, we provide some basic knowledge. In Section 3, we give some essential functions and present the general encoding methodology.  In Section 4 and Section 5, we respectively encode $AF$s to $PL_3$s and $PL_{[0,1]}$s. Section 6 is the conclusion. In the appendices, we also provide formal expressions and proofs of model equivalence between complete semantics and encoded semantics associated with $PL_2$.

\section{Preliminaries}
In this section, we will give some basic knowledge about $AF$s, $PLS$s, the model checking approach and the real equational approach.
\subsection{Argumentation frameworks and argumentation semantics}
\begin{definition}
	An \textit{argumentation framework} $(AF)$ is a pair $(A, R)$ where $A$ is a finite set of arguments and $R \subseteq A \times A$ is the attack relation.
\end{definition}

\begin{definition}
	The set of all $AF$s is denoted by $\mathcal{AF}$, and its \textit{universe} $\mathcal{A}$ is the set of all arguments across all $AF$s.
\end{definition}

\begin{definition}
	$\mathcal{A}$ is called \textit{finitary} if $A\subseteq \mathcal{A}$ is finite for each $AF=(A, R)$.
\end{definition}
We only study the case of finitary $\mathcal{A}$ in this paper. We will give some necessary background for argumentation semantics below.
\begin{definition}
	The \textit{extension-based semantics} of $\mathcal{AF}$ is a function $\mathfrak{ES}: \mathcal{AF} \to 2^{2^\mathcal{A}}$. For an $AF\in \mathcal{AF}$, each member in $\mathfrak{ES}(AF)$ is called an extension of the $AF$ under semantics $\mathfrak{ES}$.
\end{definition}
According to \cite{ref 1}, some classical semantics have been defined. For a given $AF=(A, R)$ and a set $S\subseteq A$, we say that:
\begin{itemize}
	\item $S$ is \textit{conflict-free} iff $\nexists a, b \in S$ such that $(a, b) \in R$;
	\item $a \in A$ is \textit{acceptable w.r.t. $S$} iff $\forall b \in A$ with $(b, a) \in R$, $\exists c \in S$ with $(c, b) \in R$;
	\item $S$ is \textit{admissible} iff it is conflict-free and $\forall a \in S$, $a$ is acceptable w.r.t. $S$;
	\item $S$ is a \textit{complete extension} iff it is admissible and $\forall a \in A$: if $a$ is acceptable w.r.t. $S$ then $a \in S$;
	\item A \textit{preferred extension} is a maximal (w.r.t. $\subseteq$) complete extension of the $AF$;
	\item A \textit{grounded extension} is the least (w.r.t. $\subseteq$) complete extension of the $AF$;
	\item A \textit{stable extension} is a complete extension such that $\forall a \in A \setminus S$, $\exists b \in S$ with $(b, a) \in R$.
\end{itemize}
Complete semantics is fundamental among Dung's classical semantics, because it can induce preferred semantics, grounded semantics and stable semantics.

\begin{definition}
Let $AF=(A, R)$ be an argumentation framework and $L$ be a set of labels. A \textit{labelling} of the $AF$ is a function $lab: A \to L$. The set of all labellings of the $AF$ is denoted by $L^A$.
\end{definition}

\begin{definition}
The \textit{labelling-based semantics} of $\mathcal{AF}$ is a function $\mathfrak{LS}: \mathcal{AF} \to 2^{\mathcal{LAB}}$, where $\mathcal{LAB}$ is the set of all labellings of all $AF$s.
\end{definition}

In the case of $L = \{in, out, und\}$, the authors of \cite{ref 46, ref 47} introduce concepts of legal labels and proves that extension-based semantics and labelling-based semantics are equivalent via Lab2Ext and Ext2Lab. Moreover, adopting the same approach as \cite{ref 48}, i.e., by setting $in = 1$, $out = 0$, and $und = \frac{1}{2}$, we use numerical labels $L = \{1, 0, \frac{1}{2}\}$ instead of $L = \{in, out, und\}$ in this paper. Similar to the complete extension, a complete labelling provides an equivalent characterization of complete semantics as follows.
\begin{definition}\label{defn7}
	For a given $AF = (A, R)$, a \textit{complete labelling} is a function $lab: A \to \{1, 0, \frac{1}{2}\}$ satisfying for all $a \in A$:
	\begin{equation*}
		lab(a) = 
		\begin{cases} 
			1 & \text{iff } \nexists (b,a) \in R \text{ or } \forall b_i((b_i,a) \in R): lab(b_i) = 0; \\ 
			0 & \text{iff } \exists b_j((b_j,a) \in R): lab(b_j) = 1; \\ 
			\frac{1}{2} & \text{otherwise}.
		\end{cases}
	\end{equation*}
\end{definition}

\subsection{Propositional logic systems ($PLS$s)}
For detailed background on $PLS$s, especially $PL_{[0,1]}$s, see \cite{ref 49, ref 50, ref 51, ref 52, ref 53, ref 54}. We present only essential concepts here.
\begin{definition}
	A \textit{propositional logic system} is a pair $PLS=(\mathcal{L}, \mathcal{S})$ of a \textit{language} (\textit{syntax}) $\mathcal{L}$ and a \textit{structure} (\textit{semantics}) $\mathcal{S}$ specified as follows: 
	\begin{itemize}
		\item The language $\mathcal{L}$ is a pair $(\mathcal{F}, \mathcal{C})$, where the \textit{universe} $\mathcal{F}$ is a finite or countably infinite set of atomic formulas and $\mathcal{C}$ is a tuple of connectives. 
		\item The structure $\mathcal{S}$ is a pair $(\mathcal{D}, \mathcal{M})$, where the \textit{domain} $\mathcal{D}$ is the set of truth values, and the tuple $\mathcal{M}$ consists of the interpretations of connectives in $\mathcal{C}$.
	\end{itemize}
\end{definition}

The set $\mathcal{F_{PL}}$ of all well-formed formulas (briefly, formulas) is inductively constructed from atomic formulas and connectives. In this paper, for each propositional logic system $PLS$, we let $\mathcal{F} = \mathcal{A}$, $\mathcal{D} \subseteq [0,1]$, and $\mathcal{C} = \{\neg, \wedge, \vee, \rightarrow, \leftrightarrow\}$. Propositional connectives $\neg, \wedge, \vee, \rightarrow$, and $\leftrightarrow$ represent negation, conjunction, disjunction, implication, and biimplication, respectively. These connectives are not usually independent, i.e., one connective may be defined by others in a particular $PLS$. For example, in this paper $\leftrightarrow$ is always defined as $a \leftrightarrow b = (a \rightarrow b) \wedge (b \rightarrow a)$, where $a, b \in \mathcal{F_{PL}}$. 

\begin{definition}
	For a given $PLS$, an \textit{assignment} on $\mathcal{F}_{\varphi}$ is a function $\|\cdot\|: \mathcal{F}_{\varphi} \to \mathcal{D}$, where $\mathcal{F}_{\varphi}$ is the set of all atomic formulas occurring in $\varphi$.
\end{definition}

\begin{remark}
An assignment $\|\cdot\|$ on $\mathcal{F}_{\varphi}$ is extended uniquely to an evaluation of the set of all formulas on $\mathcal{F}_{\varphi}$ via interpretations of connectives in a given $PLS$. Without causing confusion, we also denote the evaluation by $\|\cdot\|$. Note that $\bot$ and $\top$ belong to $\mathcal{F}$ and we have that $\|\top\| = 1$ and $\|\bot\| = 0$ in any $PLS$. Also note that the empty conjunction is $\top$ and the empty disjunction is $\bot$.
\end{remark}

\begin{definition}
	A model of a formula $\phi$ in a given $PLS$ is an assignment $\|\cdot\|$ such that the evaluation of $\phi$ equals 1, i.e., $\|\phi\|=1$.
\end{definition}


In this paper, we employ $PL_2$, $PL_3$s, and $PL_{[0, 1]}$s.  We assume that readers are familiar with the semantic interpretation of these connectives for the $PL_2$; their interpretations for $PL_3$s and $PL_{[0, 1]}$s will be specified below.
\subsubsection{$3$-valued propositional logic systems ($PL_3$s)}
The domain $\mathcal{D}$ of a $3$-valued propositional logic system is $\{0,1,\frac{1}{2}\}$. In this paper we mainly consider two kinds of $PL_3$s, Kleene's (strong) $3$-valued propositional logic system ($PL_3^K$) and {\L}ukasiewicz's $3$-valued propositional logic system ($PL_3^L$). Connectives $\neg$, $\wedge$, and $\vee$ are interpreted as the same meaning in the two $PL_3$s, i.e., for propositional formulas $a$ and $b$, we have
$\|\neg a\|=1-\|a\|$, $\|a\wedge b\|=\min\{\|a\|, \|b\|\}$, and $\|a\vee b\|=\max\{\|a\|, \|b\|\}$, where $\|\cdot\|$ is an evaluation on the formula set.
However, the interpretation of connective $\rightarrow$ is different in $PL_3^K$ and $PL_3^L$.
For $\|a\rightarrow b\|=\|a\|\Rightarrow^K\|b\|$ in $PL_3^K$, the truth-value table of operation $\Rightarrow^K$ is listed in Table \ref{Kleene}. For $\|a\rightarrow b\|=\|a\|\Rightarrow^L\|b\|$ in $PL_3^L$, the truth-value table of operation $\Rightarrow^L$ is listed in Table \ref{Lukasiewicz}.
\begin{table}[h]
	\caption{Truth-value table of $\Rightarrow^K$ in $PL_3^K$}\label{Kleene}%
	\begin{tabular}{@{}llll@{}}
		\toprule
		$\Rightarrow^K$ & 0 & $\frac{1}{2}$ & 1\\
		\midrule
		0 & 1 & 1 & 1 \\
		$\frac{1}{2}$ & $\frac{1}{2}$ & $\frac{1}{2}$ & 1 \\
		1 & 0 & $\frac{1}{2}$ & 1 \\
		\botrule
	\end{tabular}
\end{table}
\begin{table}[h]
	\caption{Truth-value table of $\Rightarrow^L$ in $PL_3^L$}\label{Lukasiewicz}%
	\begin{tabular}{@{}llll@{}}
		\toprule
		$\Rightarrow^L$ & 0 & $\frac{1}{2}$ & 1 \\
		\midrule
		0 & 1 & 1 & 1 \\
		$\frac{1}{2}$ & $\frac{1}{2}$ & 1 & 1 \\
		1 & 0 & $\frac{1}{2}$ & 1 \\
		\botrule
	\end{tabular}
\end{table}
The interpretation of the connective $\leftrightarrow$ in $PL_3^L$ or $PL_3^K$ can be derived from its definition and is thus omitted here.

\subsubsection{Fuzzy propositional logic systems ($PL_{[0,1]}$s)}
The domain $\mathcal{D}$ of a fuzzy propositional logic system is $[0, 1]$. Let us review essential fuzzy operations in $PL_{[0,1]}$. Note that we use $a,b,c\dots$ to denote formulas and use $u,v,w,x,y,z\dots$ to denote numbers from $[0,1]$.

\begin{definition}
	A \textit{negation} $N: [0, 1] \to [0, 1]$ satisfies $N(0) = 1$, $N(1) = 0$, and is non-increasing ($x \leqslant y \Rightarrow N(x) \geqslant N(y)$). A negation is \textit{strict} if it is strictly decreasing and continuous. The \textit{standard negation} is defined by $N(x) = 1 - x$.
\end{definition}

\begin{definition}\label{defn12}
	A \textit{triangular norm} (\textit{t-norm}) $T: [0, 1]^2 \to [0, 1]$ satisfies:
	\begin{itemize}
		\item $T (x, 1) = x$;
		\item $T (x, y) = T (y, x)$;
		\item $T (x, T (y, z)) = T (T (x, y), z)$;
		\item If $x\leqslant u$ and $y \leqslant v$, then $T (x, y) \leqslant T (u, v)$.
	\end{itemize}
	
	A t-norm $T$ is called \textit{zero-divisor-free} if it satisfies:
 \begin{equation*}
	\forall a, b \in (0,1], \quad T(a, b) > 0.
 \end{equation*}
	Equivalently, $ T $ has no zero divisors, if
		$\nexists \, a, b \in (0,1]$ such that $T(a, b) = 0$.
\end{definition}
	
We also denote $T(x,y)$ by $x \ast y$. Key continuous t-norms include:
\begin{itemize}
	\item G\"{o}del: $T_G(x, y) = \min\{x, y\}$;
	\item {\L}ukasiewicz: $T_L(x, y) = \max\{0, x + y - 1\}$;
	\item Product: $T_P(x, y) = x \cdot y$.
\end{itemize}

\begin{definition}
	A \textit{fuzzy implication} $I: [0, 1]^2 \to [0, 1]$ satisfies:
	\begin{itemize}
		\item $x \leqslant y \Rightarrow I(x, z) \geqslant I(y, z)$;
		\item $y \leqslant z \Rightarrow I(x, y) \leqslant I(x, z)$;
		\item $I(1, 0) = 0$, $I(0, 0) = I(1, 1) = 1$.
	\end{itemize}
\end{definition}

\begin{definition}
	For a t-norm $\ast$, the \textit{residual implication} ($R$-implication) is $I_\ast(x, y) = \sup\{z \mid T(x, z) \leqslant y\}$.
\end{definition}
For the R-implication $I_\ast$, it holds that:
\begin{equation*}
	I_\ast(x, y) = 1 \Longleftrightarrow x \leqslant y
\end{equation*}
where $\Longleftrightarrow$ denotes``if and only if".
Therefore, it also holds that:
\begin{equation*}
	I_\ast(x, y)=I_\ast(y, x)= 1 \Longleftrightarrow  x = y.
\end{equation*}

Residual implications for $T_G$, $T_L$, and $T_P$ are:
\begin{itemize}
	\item $I_{T_G}(x, y) = \begin{cases}
		 1 & x \leqslant y \\ 
		 y & x > y \end{cases}$; 
	\item $I_{T_L}(x, y) = \min\{1 - x + y, 1\}$; 
	\item $I_{T_P}(x, y) = \begin{cases} 1 & x \leqslant y \\ \frac{y}{x} & x > y \end{cases}$.
\end{itemize}

In this paper, we let the interpretations of some connectives in each $PL_{[0,1]}$ be as follows:
\begin{itemize}
	\item $\|\neg a\| = N(\|a\|)$;
	\item $\|a \wedge b\| = T(\|a\|, \|b\|) = \|a\| \ast \|b\|$;
	\item $\|a \rightarrow b\| = I_\ast(\|a\|, \|b\|)$;
	\item $\|a \leftrightarrow b\| = \|(a \rightarrow b) \wedge (b \rightarrow a)\| = I_\ast(\|a\|, \|b\|) \ast I_\ast(\|b\|, \|a\|)$.
\end{itemize}

A $PL_{[0,1]}$ equipped with a continuous negation $N^\star$, a continuous t-norm and an $R$-implication is denoted by $PL_{[0, 1]}^\star$. A $PL_{[0,1]}$ equipped with a strict negation $N^{\star\star}$, a continuous t-norm and an $R$-implication is denoted by $PL_{[0, 1]}^{\star\star}$. A $PL_{[0,1]}$ equipped with a standard negation, a G\"{o}del t-norm $T_G$ and an $R$-implication $I_{T_G}$ is denoted by $PL_{[0, 1]}^G$. A $PL_{[0,1]}$ equipped with a standard negation, a {\L}ukasiewicz t-norm $T_L$ and an $R$-implication $I_{T_L}$ is denoted by $PL_{[0, 1]}^L$. A $PL_{[0,1]}$ equipped with a standard negation, a Product t-norm $T_P$ and an $R$-implication $I_{T_P}$ is denoted by $PL_{[0, 1]}^P$. 
\begin{remark}
The interpretation of the symbol ``$\wedge$'' varies in the literature. For example, Reference \cite{ref 51} uses ``$\&$'' for t-norm conjunction ($\|a \& b\| = T(\|a\|, \|b\|)$) and ``$\wedge$'' for minimum conjunction ($\|a \wedge b\| = \min\{\|a\|, \|b\|\}$). In contrast, Reference \cite{ref 54} employs ``$\wedge$'' to denote t-norm conjunction ($\|a \wedge b\| = T(\|a\|, \|b\|)$). This paper adopts the latter convention, using ``$\wedge$'' to represent t-norm conjunction. $PL_{[0,1]}^G$ and $PL_{[0,1]}^P$ are first introduced in \cite{ref 53}. They are variants of the G\"{o}del logic system and the Product logic system in \cite{ref 51}, respectively. $PL_{[0,1]}^L$ is the same as the {\L}ukasiewicz logic system in \cite{ref 51}.
\end{remark}
\subsection{The model checking approach}\label{Checking}
Besnard and Doutre \cite{ref 29} present a model checking approach that encodes $AF$s as propositional formulas in $PL_2$. The models of each formula correspond bijectively to extensions of an $AF$ under given semantics. For a given $AF$, each model's core—defined as the set of arguments assigned the value 1—constitutes an extension. The forms of formulas encoded from $AF$s vary by semantics. Below we list key results from \cite{ref 29}, where $AF = (A, R)$:

1. \textbf{Conflict-free semantics}: $S \subseteq A$ is conflict-free iff it is the core of a model of:
\begin{equation*}
	\bigwedge_{a\in A}(a \rightarrow \bigwedge_{(b,a)\in R} \neg b)
\end{equation*}

2. \textbf{Stable semantics}: $S \subseteq A$ is a stable extension iff it is the core of a model of:
\begin{equation*}
	\bigwedge_{a\in A}(a \leftrightarrow \bigwedge_{(b,a)\in R} \neg b)
\end{equation*}

3. \textbf{Admissible semantics}: $S \subseteq A$ is admissible iff it is the core of a model of:
\begin{equation*}
	\bigwedge_{a\in A}( (a \rightarrow \bigwedge_{(b,a)\in R} \neg b) \wedge (a \rightarrow \bigwedge_{(b,a)\in R} \bigvee_{(c,b)\in R}c) )
\end{equation*}

4. \textbf{Complete semantics}: $S \subseteq A$ is a complete extension iff it is the core of a model of:
\begin{equation*}
	\bigwedge_{a\in A}( (a \rightarrow \bigwedge_{(b,a)\in R} \neg b) \wedge (a \leftrightarrow \bigwedge_{(b,a)\in R} \bigvee_{(c,b)\in R}c) )
\end{equation*}
Note that the empty conjunction is $\top$ and the empty disjunction is $\bot$. 

The authors only provide an analysis of the above propositions in \cite{ref 29}. Formal expressions and rigorous proofs associated with complete semantics are shown in Appendix \ref{secA1}.
\subsection{The real equational approach}
 Gabbay proposes the real equational approach in \cite{ref 12}. 
\begin{definition}\label{defn15}
	A \textit{real equational function} in $k$ variables $\{x_1, x_2, \dots, x_k\}$ over the real interval $[0,1]$ is a continuous function $h: [0, 1]^k \to [0, 1]$ satisfying:
	\begin{enumerate}
		\item \textbf{Boundary Conditions:}
		\begin{itemize}
			\item[(a)] $h(0, 0, \dots, 0) = 1$,
			\item[(b)] $h(x_1, \dots, x_{i-1}, 1, x_{i+1}, \dots, x_k) = 0$ for all $i \in \{1, \dots, k\}$ and all $x_j \in [0,1]$ ($j \neq i$).
		\end{itemize}
		\item \textbf{Symmetry:}
		\begin{itemize}
			\item[(c)] $h(x_1, \dots, x_k) = h(x_{\sigma(1)}, \dots, x_{\sigma(k)})$ for all $(x_1, \dots, x_k) \in [0,1]^k$ and all permutations $\sigma$ of $\{1, \dots, k\}$.
		\end{itemize}
	\end{enumerate}
\end{definition}
\begin{definition}
	An \textit{equational argumentation network} over $[0, 1]$ has the form $(A, R, h_a)$, where\\
	(a) $(A, R)$ is a Dung-style $AF$.\\
	(b) For each $a \in A, h_a$ is a $k_a$-ary real equational function, where $k_a$ is the number of attackers of $a$.\\
	(c) For some $a \in A$, if $\nexists b((b,a)\in R)$, then $h_a=1$.
\end{definition}
	 In \cite{ref 12}, $A$ is defined as a finite or infinite set; however, we restrict it to a finite set in this paper.
\begin{definition}
	A \textit{real equational extension} is a function $\|\cdot\|: A\to [0, 1]$ such that $\forall a\in A$,
	\begin{equation*}
		\|a\| =\begin{cases}
			1, &\nexists c((c,a)\in R),\\
			h_a(\|b_1\|, \dots , \|b_{k_a}\|), &\text{otherwise},
		\end{cases}
	\end{equation*}
	where $h_a$ is a $k_a$-ary real equational function and $\{b_1, \dots , b_{k_a} \}$ is the set of all attackers of $a$.
\end{definition}
The graphical structure of an equational argumentation network is the same as the structure of a classical $AF$. In this paper, for the sake of clarity and without altering the essence of the equational argumentation network $(A, R, h_a)$, we treat an equational argumentation network $(A, R, h_a)$ as a classical $AF$ associated with a \textit{real equational system} $Eq^R$, where $Eq^R$ is the equational system as defined in the real equational extension.
Each real equational system $Eq^R$ corresponds to a \textit{real equational semantics}, and without causing confusion, we also denote this real equational semantics by $Eq^R$. Therefore, we still denote the equational argumentation network as ``$AF$".
A real equational extension is equivalent to a labelling that satisfies the given real equational semantics. If the continuity condition is omitted from the real equational function, the resulting function is termed an \textit{equational function}, and the resulting semantics—an \textit{equational semantics}—derives from an \textit{equational system}.
	
Gabbay gives four specific equational systems in \cite{ref 12}. For a given $AF=(A, R)$ and $a\in A$, an important one is $Eq_{\text{inverse}}^R$: 
	\begin{equation*}
	\|a\| =\begin{cases}
		1, &\nexists c((c,a)\in R),\\
		\prod_{i=1}^{k_a}(1-\|b_i\|), &\text{otherwise},
	\end{cases}
\end{equation*}
and another important one is $Eq_{\max}^R$: 
\begin{equation*}
	\|a\| =\begin{cases}
		1, &\nexists c((c,a)\in R),\\
		1 - \max_{i=1}^{k_a}\|b_i\|, &\text{otherwise},
	\end{cases}
\end{equation*}
where $\{b_1, \dots , b_{k_a} \}$ is the set of all attackers of $a$. 
\section{The general encoding methodology}\label{sec3}
In this section, we give some definitions that are used in this paper and introduce the general encoding methodology.

\begin{definition}
	Let $lab: A \to L$ be a labelling of an $AF=(A, R)$.
    The labelling $lab$ is called a \textit{numerical labelling} if $L\subseteq[0,1]$.
    The labelling $lab$ is called a \textit{$2$-valued labelling} if $L=\{0, 1\}$.
    The labelling $lab$ is called a \textit{$3$-valued labelling} if $L=\{0, \frac{1}{2}, 1\}$.
    The labelling $lab$ is called a \textit{$[0,1]$-valued labelling} if $L=[0,1]$.
\end{definition}
For an $AF=(A, R)$ and $a\in A$, let $lab_{num}$ be a numerical labelling, $LAB_{num}$ be a set of numerical labellings, $lab_2$ be a $2$-valued labelling, $LAB_2$ be a set of $2$-valued labellings, $lab_3$ be a $3$-valued labelling, and $LAB_3$ be a set of $3$-valued labellings. Then we give the definitions of the binarization and the ternarization of $LAB_{num}$.
\begin{definition}
	The \textit{binarization} of $LAB_{num}$ is a total function $T_2: LAB_{num}\to LAB_2$, $lab_{num} \mapsto lab_2$, such that 
	\begin{equation*}
		T_2(lab_{num})(a)=lab_2(a) =
		\begin{cases}
			1 & lab_{num}(a)=1 
			\\0 & lab_{num}(a)\neq1
		\end{cases}.
	\end{equation*}
	$T_2(lab_{num})$ is called the \textit{binarized labelling} from the labelling $lab_{num}$.
\end{definition}

\begin{definition}\label{defn21}
	A \textit{ternarization} of $LAB_{num}$ is a partial function $T_{\text{com}}: LAB_{num}\to LAB_3$, $lab_{num} \mapsto lab_3$ s.t. 
	\begin{equation*}
		T_{\text{com}}(lab_{num})(a)=
		\begin{cases}
			1 & \text{iff } lab_{num}(a)=1, \\
			0 & \text{iff } \exists b_i((b_i,a)\in R): T_{\text{com}}(lab_{num})(b_i)=lab_{num}(b_i)=1,\\
			\frac{1}{2} & \text{otherwise}.
		\end{cases}
	\end{equation*}
		$T_{\text{com}}(lab_{num})$ is called the \textit{ternarized labelling} from the labelling $lab_{num}$.
\end{definition}
 Note that if for any labelling $lab_{num}$ in $LAB_{num}$ there does not exist $(b, a)\in R$ with $lab_{num}(a)=lab_{num}(b)=1$, then $T_{\text{com}}$ is a total function on $LAB_{num}$. Therefore, for the model set of $ec_1(AF)$ (or $ec_2(AF)$) in $PL_2$ (or $PL_3^K$ or $PL_3^L$ or $PL_{[0,1]}$s), $T_{\text{com}}$ is always a total function. If $\nexists a\in A, lab_{num}(a)=1$, then $\forall b\in A$, $T_{\text{com}}(lab_{num})(b)=\frac{1}{2}$ from Definition \ref{defn21}.

Then we introduce the general encoding methodology by proposing the concepts of the encoding, the encoded semantics, the translating, and the translatable semantics. 
\begin{definition}
	An \textit{encoding} of $\mathcal{AF}$ is a map $ec: \mathcal{AF} \to \mathcal{F_{PL}}$ such that for any $AF = (A, R) \in \mathcal{AF}$, the set of all atomic formulas in $ec(AF)$ coincides with $A$.
\end{definition}
When we say that we encode $AF$s to a $PLS$, it means that we map $\mathcal{AF}$ to $\mathcal{F_{PL}}$ via an encoding. 
Building on existing mappings \cite{ref 29} and assuming that an empty conjunction is $\top$ and an empty disjunction is $\bot$, we formalize and name two key encodings:
\begin{itemize}
	\item The \textit{normal encoding} is $ec_1: \mathcal{AF} \to \mathcal{F_{PL}}$, s.t. 
	\begin{equation*}
		ec_1(AF)= \bigwedge_{a\in A}(a\leftrightarrow\bigwedge_{(b,a)\in R}\neg b).
	\end{equation*}
	
	\item The \textit{regular encoding} is $ec_2: \mathcal{AF} \to \mathcal{F_{PL}}$, s.t. 
	\begin{equation*}
		ec_2(AF)= \bigwedge_{a\in A}((a\rightarrow\bigwedge_{(b,a)\in R}\neg b)\wedge(a\leftrightarrow\bigwedge_{(b,a)\in R}\bigvee_{(c,b)\in R}c)).
	\end{equation*}
\end{itemize}	

\begin{definition}
	For an $AF=(A, R)$, an \textit{assignment} of $A$ is a function $\|\cdot\|: A \to L$, where $L\subseteq[0,1]$. 
\end{definition}
\begin{remark}
	From this definition, an assignment of $A$ is equivalent to a numerical labelling of the $AF=(A, R)$ and is equivalent to an assignment on $\mathcal{F}_{ec(AF)}$. We shall not distinguish between these three notions; instead, we simply refer to them as an ``assignment" and employ the unified symbol $\|\cdot\|$ for them. In addition, for an assignment $\|\cdot\|$, $T_2(\|\cdot\|)$ is denoted by $\|\cdot\|_2$ and $T_{\text{com}}(\|\cdot\|)$ is denoted by $\|\cdot\|_{\text{com}}$ in this paper.
\end{remark}

\begin{definition}
	An assignment $\|\cdot\|$ is called a \textit{model} of an $AF=(A, R)$ under a given semantics $\mathfrak{LS}$, if $\|\cdot\|\in\mathfrak{LS}(AF)$.
\end{definition}
\begin{definition}
For an $AF=(A, R)$, an encoding $ec$ and a $PLS$, the \textit{encoded semantics} of $\mathcal{AF}$ is a function $\mathfrak{LS}_{ec}: \mathcal{AF} \to 2^{\mathcal{LAB}}$ such that $\mathfrak{LS}_{ec}(AF)=\{\|\cdot\|\mid\|ec(AF)\|=1\}$. 
\end{definition}
\begin{definition}
The encoded semantics is a \textit{fuzzy normal encoded semantics} if $ec$ is the normal encoding and the associated $PLS$ is a $PL_{[0,1]}$. The encoded semantics is a \textit{continuous fuzzy normal encoded semantics} if it is a fuzzy normal encoded semantics and the negation and the t-norm in $PL_{[0,1]}$ are continuous. The encoded semantics is a \textit{strict fuzzy normal encoded semantics} if it is a fuzzy normal encoded semantics and the associated $PL_{[0,1]}$ is equipped with a strict negation and a continuous t-norm.
\end{definition}
Note that for a strict fuzzy normal encoded semantics, we require the negation to be strictly monotonic but do not require the t-norm to be so. The reason is that some well-known t-norms—such as Gödel t-norm $T_G$ and Łukasiewicz t-norm $T_L$—are not strictly monotonic, whereas strict negations can capture a characteristic of argumentation semantics: for an argument $a$, the value of $\neg a$ strictly decreases as the value of $a$ increases.

The general encoding methodology is an approach to induce semantics for $\mathcal{AF}$. For a given $PLS$, each encoding can induce a type of semantics for $\mathcal{AF}$. For a given $ec$, each $PLS$ can also induce a type of semantics for $\mathcal{AF}$. However, not every arbitrary pair of $ec$ and $PLS$ is meaningful for deriving the encoded semantics of $\mathcal{AF}$. For a given semantics of $\mathcal{AF}$, the encoding and the $PLS$ that can induce this semantics are of great importance.
\begin{definition}
For a given $PLS$ and a given semantics $\mathfrak{LS}$ of $\mathcal{AF}$, a translating $tr$ of $\mathcal{AF}$ is an encoding such that models of $tr(AF)$ in the $PLS$ are the same as models of the $AF$ under $\mathfrak{LS}$, i.e., $\{\|\cdot\|\mid\|tr(AF)\|=1\}=\mathfrak{LS}(AF)$.
\end{definition}

\begin{definition}
	For a given semantics $\mathfrak{LS}$ of $\mathcal{AF}$, the semantics $\mathfrak{LS}$ is called translatable if there exists a translating of $\mathcal{AF}$.
\end{definition} 

\section{Encoding $AF$s to $PL_3$s}
Since Dung's classical semantics can be expressed by $3$-valued labellings, a natural thought is to encode $AF$s to $3$-valued $PLS$s.
In this section, we study two encodings which are normal encoding $ec_1$ and regular encoding $ec_2$ and we consider two $3$-valued $PLS$s which are $PL_3^K$ and $PL_3^L$. So four cases need to be studied. In each case, we explore the relationship between models of an $AF$ under a given classical semantics and models of the encoded formula in a $PL_3$. Note that the label set of $AF$s is $\{0,1,\frac{1}{2}\}$.

\subsection{Encoding $AF$s to $PL_3$s by the normal encoding}
 When we encode $AF$s to $PL_3^K$ and $PL_3^L$ by the normal encoding $ec_1$, we will get the related encoded semantics. First, we investigate the relationship between stable semantics and the encoded semantics associated with $ec_1$ and $PL_3^K$ by characterizing the model relationship.
\begin{theorem}
	An assignment is a model of an $AF$ under stable semantics iff it is a model of $ec_1(AF)$ in $PL_3^K$.
\end{theorem}
\begin{proof}
	An assignment $\|\cdot\|$ is a model of $ec_1(AF)$ in $PL_3^K$
	\\$\Longleftrightarrow$ $\|\bigwedge_{a\in A}(a\leftrightarrow\bigwedge_{(b_i,a)\in R}\neg{b_i})\|=1$
	\\$\Longleftrightarrow$ by Table \ref{Kleene}, $\forall a\in A$:
	\begin{equation*}
		\|a\|=\|\bigwedge_{(b_i,a)\in R}\neg{b_i}\|=1
	\end{equation*}
	or
	\begin{equation*}
		\|a\|=\|\bigwedge_{(b_i,a)\in R}\neg{b_i}\|=0.
	\end{equation*}
	We need to discuss two cases.
	\begin{itemize}
		\item Case 1, $\|a\|=1$.
		\\$\|a\|=1$
		\\$\Longleftrightarrow$ $\|\bigwedge_{(b_i,a)\in R}\neg{b_i}\|=1$
		\\$\Longleftrightarrow$ $\nexists c((c,a)\in R)$ or $\forall b_i((b_i,a)\in R)$: $\|\neg{b_i}\|=1$
		\\$\Longleftrightarrow$ $\nexists c((c,a)\in R)$ or $\forall b_i((b_i,a)\in R)$: $\|b_i\|=0$. 
		\\$\Longleftrightarrow$ $\|a\|=1$ satisfies complete semantics.
		\item Case 2, $\|a\|=0$.
		\\$\|a\|=0$
		\\$\Longleftrightarrow$ $\|\bigwedge_{(b_i,a)\in R}\neg{b_i}\|=0$
		\\$\Longleftrightarrow$ $\exists b_i((b_i,a)\in R)$: $\|\neg{b_i}\|=0$
		\\$\Longleftrightarrow$ $\exists b_i((b_i,a)\in R)$: $\|b_i\|=1$. 
		\\$\Longleftrightarrow$ $\|a\|=0$ satisfies complete semantics.
	\end{itemize}
	Thus, an assignment $\|\cdot\|$ is a model of $ec_1(AF)$ in  $PL_3^K$ 
	\\$\Longleftrightarrow$ $\|\cdot\|$ is a model of the $AF$ under complete semantics and $\forall a\in A$: $\|a\|\in\{0,1\}$ 
	 \\$\Longleftrightarrow$ $\|\cdot\|$ is a model of the $AF$ under stable semantics.
\end{proof}

\begin{corollary}
	Stable semantics is translatable by $ec_1$ and $PL_3^K$.
\end{corollary}

While stable semantics was previously shown to be translatable via $ec_1$ and $PL_2$ \cite{ref 29}, we demonstrate here that it is also translatable using an alternative approach based on $ec_1$ and $PL_3^K$. Next, we investigate the relationship between complete semantics and the encoded semantics associated with $ec_1$ and $PL_3^L$ by characterizing the model relationship.

\begin{theorem}\label{thm5.3}
	An assignment is a model of an $AF$ under complete semantics iff it is a model of $ec_1(AF)$ in $PL_3^L$.
\end{theorem}

\begin{proof}
		An assignment $\|\cdot\|$ is a model of $ec_1(AF)$  in $PL_3^L$
	\\$\Longleftrightarrow$ $\|\bigwedge_{a\in A}(a\leftrightarrow\bigwedge_{(b_i,a)\in R}\neg{b_i})\|=1$
	\\$\Longleftrightarrow$ by Table \ref{Lukasiewicz}, $\forall a\in A$:
	\begin{equation*}
		\|a\|=\|\bigwedge_{(b_i,a)\in R}\neg{b_i}\|=1
	\end{equation*}
	or
	\begin{equation*}
		\|a\|=\|\bigwedge_{(b_i,a)\in R}\neg{b_i}\|=0
	\end{equation*}
	or
	\begin{equation*}
		\|a\|=\|\bigwedge_{(b_i,a)\in R}\neg{b_i}\|=\frac{1}{2}.
	\end{equation*}
	We need to discuss three cases.
	\begin{itemize}
		\item Case 1, $\|a\|=1$.
		\\$\|a\|=1$
		\\$\Longleftrightarrow$ $\|\bigwedge_{(b_i,a)\in R}\neg{b_i}\|=1$
		\\$\Longleftrightarrow$ $\nexists c((c,a)\in R)$ or $\forall b_i((b_i,a)\in R)$: $\|\neg{b_i}\|=1$
		\\$\Longleftrightarrow$ $\nexists c((c,a)\in R)$ or $\forall b_i((b_i,a)\in R)$: $\|b_i\|=0$
		\\$\Longleftrightarrow$ $\|a\|=1$ satisfies complete semantics.
		\item Case 2, $\|a\|=0$.
		\\$\|a\|=0$
		\\$\Longleftrightarrow$ $\|\bigwedge_{(b_i,a)\in R}\neg{b_i}\|=0$
		\\$\Longleftrightarrow$ $\exists b_i((b_i,a)\in R)$: $\|\neg{b_i}\|=0$
		\\$\Longleftrightarrow$ $\exists b_i((b_i,a)\in R)$: $\|b_i\|=1$
		\\$\Longleftrightarrow$ $\|a\|=0$ satisfies complete semantics.
		\item Case 3, $\|a\|=\frac{1}{2}$.
		\\$\|a\|=\frac{1}{2}$
		\\$\Longleftrightarrow$ $\|a\|\neq 1$ and $\|a\|\neq 0$
		\\$\Longleftrightarrow$ $\|a\|=1$ does not satisfy complete semantics and $\|a\|=0$ does not satisfy complete semantics
		\\$\Longleftrightarrow$ $\|a\|=\frac{1}{2}$ satisfies complete semantics.
	\end{itemize}
		Thus, an assignment $\|\cdot\|$ is a model of $ec_1(AF)$ in  $PL_3^L$ iff $\|\cdot\|$ is a model of the $AF$ under complete semantics.
\end{proof}

\begin{corollary}\label{cor5.4}
	Complete semantics is translatable by $ec_1$ and $PL_3^L$.
\end{corollary}
In \cite{ref 29}, stable semantics has been shown to be translatable via $ec_1$ and $PL_2$, while complete semantics needs to be translated via $ec_2$ and $PL_2$. Here we prove that complete semantics is translatable based on $ec_1$ and $PL_3^L$, i.e., complete semantics is an encoded semantics w.r.t. $ec_1$ and $PL_3^L$. This indicates that to obtain complete semantics, we need not resort to the more formally complex $ec_2$, but can instead utilize the more concise formalism $ec_1$ while seeking an appropriate logic system. Consequently, this naturally suggests that employing $ec_1$ with proper $PLS$s yields $[0,1]$-valued semantics. We will pursue this approach in Section 5, where we demonstrate that this $[0,1]$-valued encoded semantics is closely related to, yet distinct from, Gabbay's real equational semantics. 

According to the truth tables of $PL_3$s \cite{ref 52}, 
Heyting's $PL_3$ and Reichenbach's $PL_3$ yield identical conclusions for complete semantics to those established in Theorem \ref{thm5.3} and Corollary \ref{cor5.4}.

\subsection{Encoding $AF$s to $PL_3$s by the regular encoding}
If we encode $AF$s to $PL_3^K$ and $PL_3^L$ by the regular encoding $ec_2$, we will obtain the related encoded semantics. First, we investigate the relationship between complete semantics and the encoded semantics associated with $ec_2$ and $PL_3^K$ by characterizing the model relationships via $T_2$ and $T_{\text{com}}$.
\begin{theorem}
	If an assignment $\|\cdot\|$ is a model of an $AF$ under complete semantics, then $T_2(\|\cdot\|)$ is a model of $ec_2(AF)$ in $PL_3^K$. 
\end{theorem}

\begin{proof}
	According to Theorem \ref{thm4.2}, if $\|\cdot\|$ is a model of an $AF$ under complete semantics, then $T_2(\|\cdot\|)$ is a model of $ec_2(AF)$ in $PL_2$. Since $PL_3^K$ is equivalent to $PL_2$ when the truth values are restricted to \{0, 1\}, $T_2(\|\cdot\|)$ is a model of $ec_2(AF)$ in $PL_3^K$.
\end{proof}

\begin{theorem}
If an assignment $\|\cdot\|$ is a model of $ec_2(AF)$ in $PL_3^K$, then $T_{\text{com}}(\|\cdot\|)$ is a model of the $AF$ under complete semantics.
\end{theorem}
	\begin{proof}	
	An assignment $\|\cdot\|$ is a model of $ec_2(AF)$ in $PL_3^K$
	\\$\Longleftrightarrow$
	\begin{equation*}
		\|\bigwedge_{a\in A}((a\rightarrow\bigwedge_{(b,a)\in R}\neg b)\wedge(a\leftrightarrow\bigwedge_{(b,a)\in R}\bigvee_{(c,b)\in R}c))\|=1
	\end{equation*}
	\\$\Longleftrightarrow$ by $PL_3^K$, $\forall a\in A$:
	\begin{equation*}
		\|a\rightarrow\bigwedge_{(b,a)\in R}\neg b\|=1
	\end{equation*}
	and 
	\begin{equation}\label{equ2}
		\|a\leftrightarrow\bigwedge_{(b,a)\in R}\bigvee_{(c,b)\in R}c\|=1.
	\end{equation}
	
	By Equation \ref{equ2}, we have $\forall a\in A$:
		\begin{equation*}
		\|a\|=\|\bigwedge_{(b,a)\in R}\bigvee_{(c,b)\in R}c\|=1
	\end{equation*}
	or
		\begin{equation*}
		\|a\|=\|\bigwedge_{(b,a)\in R}\bigvee_{(c,b)\in R}c\|=0.
	\end{equation*}
	Thus, $\forall a\in A$, $a\neq \frac{1}{2}$. Hence, $\|\cdot\|=T_2(\|\cdot\|)$ and a model $\|\cdot\|$ of $ec_2(AF)$ in $PL_3^K$ is also a model of $ec_2(AF)$ in $PL_2$. By Theorem \ref{thm4.3}, if $\|\cdot\|$ is a model of $ec_2(AF)$ in $PL_2$, then $T_{\text{com}}(\|\cdot\|)$ is a model of the $AF$ under complete semantics.
\end{proof}

Next, we investigate the relationship between complete semantics and the encoded semantics associated with $ec_2$ and $PL_3^L$ by characterizing the model relationships.

\begin{theorem}
	Each model of an $AF$ under complete semantics is a model of $ec_2(AF)$ in $PL_3^L$.
\end{theorem}
\begin{proof}
	First, an assignment $\|\cdot\|'$ is a model of $ec_2(AF)$ in $PL_3^L$
	\\$\Longleftrightarrow$
	\begin{equation*}
		\|\bigwedge_{a\in A}((a\rightarrow\bigwedge_{(b,a)\in R}\neg b)\wedge(a\leftrightarrow\bigwedge_{(b,a)\in R}\bigvee_{(c,b)\in R}c))\|'=1
	\end{equation*}
	\\$\Longleftrightarrow$ by $PL_3^L$, $\forall a\in A$:
	\begin{equation}\label{L1}
		\|a\rightarrow\bigwedge_{(b,a)\in R}\neg b\|'=1
	\end{equation}
	and 
	\begin{equation}\label{L2}
		\|a\leftrightarrow\bigwedge_{(b,a)\in R}\bigvee_{(c,b)\in R}c\|'=1.
	\end{equation}
	
	Second, let $\|\cdot\|$ be any model of the $AF$ under complete semantics. Then we need to discuss three cases.
	\begin{itemize}
		\item Case 1, $\|a\|=1$.
		\\$\|a\|=1$
		\\$\Longleftrightarrow$ by complete semantics, $\nexists b((b,a)\in R)$ or $\forall b_i ((b_i, a)\in R): \|b_i\|=0$ 
		\\$\Longleftrightarrow$ by complete semantics, $\nexists b((b,a)\in R)$ or $\forall b_i ((b_i, a)\in R): [\|b_i\|=0$ and $\exists c_j ((c_j, b_i)\in R): \|c_j\|=1$] 
		\\$\Longleftrightarrow$ Equations \ref{L1} and \ref{L2} hold under $\|\cdot\|$.		
		\item Case 2, $\|a\|=0$.
		\\$\|a\|=0$
		\\$\Longleftrightarrow$ by complete semantics, $\exists b_i ((b_i, a)\in R): \|b_i\|=1$
		\\$\Longleftrightarrow$ by complete semantics, $\exists b_i ((b_i, a)\in R): [\|b_i\|=1$] and [$\nexists c((c, b_i)\in R)$ or $\forall c_j((c_j, b_i)\in R): \|c_j\|=0$]
		\\$\Longrightarrow$ Equations \ref{L1} and \ref{L2} hold under $\|\cdot\|$.
		\item Case 3, $\|a\|=\frac{1}{2}$.
		\\$\|a\|=\frac{1}{2}$
		\\$\Longleftrightarrow$ by complete semantics, [$\forall b_i ((b_i, a)\in R): \|b_i\|\neq1$] and [$\exists b_i ((b_i, a)\in R): \|b_i\|=\frac{1}{2}$]
		\\$\Longrightarrow$ by complete semantics, [$\forall b_i ((b_i, a)\in R)\exists c((c,b_i)\in R)$: $\|b_i\|\neq1$ and $\|c\|\neq 0$] and [$\exists b_i ((b_i, a)\in R)\exists c_i((c_i,b_i)\in R)\nexists c_j((c_j,b_i)\in R): \|b_i\|=\frac{1}{2}$ and $\|c_i\|= \frac{1}{2}$ and $\|c_j\|= 1$]
		\\$\Longrightarrow$ Equations \ref{L1} and \ref{L2} hold under $\|\cdot\|$.
	\end{itemize}
	From the three cases, since for each $a\in A$ Equations \ref{L1} and \ref{L2} hold under $\|\cdot\|$, a model $\|\cdot\|$ of the $AF$ under complete semantics is a model of $ec_2(AF)$ in $PL_3^L$.
\end{proof}
In fact, we have a proposition that a model of $ec_2(AF)$ in $PL_3^L$ may not be a model of the $AF$ under complete semantics. The following example demonstrates this proposition:
\begin{example}
	Let $AF=(A, R), A=\{a, b\}, R=\{(a, b), (b, a)\}$. The assignment $\|a\|=\frac{1}{2}$ and $\|b\|=0$ satisfies $ec_2(AF)$ in $PL_3^L$ but does not satisfy complete semantics. 
\end{example}

By function $T_{\text{com}}$, each model of $ec_2(AF)$ in $PL_3^L$ is turned into a model of the $AF$ under complete semantics. The lemma and theorem are presented below.
\begin{lemma}\label{lem5.7}
	For a given $AF=(A,R)$, if an assignment $\|\cdot\|$ is a model of $ec_2(AF)$ in $PL_3^L$, then $T_2(\|\cdot\|)$ is a model of $ec_2(AF)$ in $PL_2$.
\end{lemma}
\begin{proof}
	An assignment $\|\cdot\|$ is a model of $ec_2(AF)$ in $PL_3^L$
	\\$\Longleftrightarrow$ 
	\begin{equation*}
		\|\bigwedge_{a\in A}((a\rightarrow\bigwedge_{(b,a)\in R}\neg b)\wedge(a\leftrightarrow\bigwedge_{(b,a)\in R}\bigvee_{(c,b)\in R}c))\|=1,
	\end{equation*}
	\\$\Longleftrightarrow$  $\forall a\in A$,
	\begin{equation}\label{eq13}
		\|a\rightarrow\bigwedge_{(b,a)\in R}\neg b\|=1
	\end{equation}
	and
	\begin{equation}\label{eq14}
		\|a\leftrightarrow\bigwedge_{(b,a)\in R}\bigvee_{(c,b)\in R}c\|=1.
	\end{equation}
	 Similarly, an assignment $\|\cdot\|'$ is a model of $ec_2(AF)$ in $PL_2$
	\\$\Longleftrightarrow$ 
	\begin{equation*}
		\|\bigwedge_{a\in A}((a\rightarrow\bigwedge_{(b,a)\in R}\neg b)\wedge(a\leftrightarrow\bigwedge_{(b,a)\in R}\bigvee_{(c,b)\in R}c))\|'=1,
	\end{equation*}
	\\$\Longleftrightarrow$  $\forall a\in A$,
	\begin{equation}\label{equ15}
		\|a\rightarrow\bigwedge_{(b,a)\in R}\neg b\|'=1
	\end{equation}
	and
	\begin{equation}\label{equ16}
		\|a\leftrightarrow\bigwedge_{(b,a)\in R}\bigvee_{(c,b)\in R}c\|'=1.
	\end{equation}
	Then, for $\forall a\in A$, we discuss three cases.
	\begin{itemize}
		\item Case 1, $\|a\|=1$. 
		\\$\|a\|=\|a\|_2=1$
		\\$\Longleftrightarrow$ $\nexists b((b, a)\in R)$ or $\forall b_i ((b_i, a)\in R)\exists c_j ((c_j, b_i)\in R)$: [$\|b_i\|=0$ (from Equation \ref{eq13})] and [$\|c_j\|=1$ (from Equation \ref{eq14})]
		\\$\Longrightarrow$ $\nexists b((b, a)\in R)$ or $\forall b_i ((b_i, a)\in R)\exists c_j ((c_j, b_i)\in R)$: [$\|b_i\|_2=0$] and [$\|c_j\|_2=1$]
		\\$\Longleftrightarrow$ Equations \ref{equ15} and \ref{equ16} hold under $\|\cdot\|_2$ in Case 1.
		\item Case 2, $\|a\|=0$. 
		\\$\|a\|=\|a\|_2=0$
		\\$\Longleftrightarrow$ $\exists b_i ((b_i, a)\in R)\nexists c ((c, b_i)\in R)$ or $\exists b_i ((b_i, a)\in R)\forall c_j ((c_j, b_i)\in R): \|c_j\|=0$ (from Equations \ref{eq13} and \ref{eq14})
		\\$\Longrightarrow$ $\exists b_i ((b_i, a)\in R)\nexists c ((c, b_i)\in R)$ or $\exists b_i ((b_i, a)\in R)\forall c_j ((c_j, b_i)\in R): \|c_j\|_2=0$ 
		\\$\Longleftrightarrow$ Equations \ref{equ15} and \ref{equ16} hold under $\|\cdot\|_2$ in Case 2.
		\item 	Case 3, $\|a\|=\frac{1}{2}$. 
		We have $\|a\|_2=0$ from $\|a\|=\frac{1}{2}$ and $T_2$. Thus Equation \ref{equ15} holds under $\|\cdot\|_2$ in Case 3. Next we need to prove Equation \ref{equ16} holds under $\|\cdot\|_2$ in this case.
		
		$\|a\|=\frac{1}{2}$ (and then $\|a\|_2=0$)
		\\$\Longrightarrow$ $\exists b_i ((b_i, a)\in R)\exists c_j ((c_j, b_i)\in R)\nexists c_k ((c_k, b_i)\in R)$: $\|c_j\|=\frac{1}{2}$ and $\|c_k\|=1$ (from Equation \ref{eq14})]
		\\$\Longrightarrow$ $\exists b_i ((b_i, a)\in R)\exists c_j ((c_j, b_i)\in R)\nexists c_k ((c_k, b_i)\in R)$: $\|c_j\|_2=0$ and $\|c_k\|_2=1$
		\\$\Longrightarrow$ $\exists b_i ((b_i, a)\in R)\forall c_j ((c_j, b_i)\in R): \|c_j\|_2=0$
		\\$\Longleftrightarrow$ Equation \ref{equ16} holds under $\|\cdot\|_2$ in Case 3.
		\end{itemize}
	From the three cases, for all $a\in A$, Equations \ref{equ15} and \ref{equ16} hold under $\|\cdot\|_2$, i.e., $\|\cdot\|_2$ is a model of $ec_2(AF)$ in $PL_2$.
\end{proof}

\begin{theorem}
	For a given $AF=(A,R)$, if $\|\cdot\|$ is a model of $ec_2(AF)$ in $PL_3^L$, then $T_{\text{com}}(\|\cdot\|)$ is a model of the $AF$ under complete semantics.
\end{theorem}
\begin{proof}
	For the $AF=(A,R)$, if $\|\cdot\|$ is a model of $ec_2(AF)$ in $PL_3^L$, then according to Lemma \ref{lem5.7}, $T_2(\|\cdot\|)$ is a model of $ec_2(AF)$ in $PL_2$.
	According to Theorem \ref{thm4.3}, $T_{\text{com}}(T_2(\|\cdot\|))$ is a model of the $AF$ under complete semantics. Since $T_{\text{com}}(T_2(\|\cdot\|))=T_{\text{com}}(\|\cdot\|)$, $T_{\text{com}}(\|\cdot\|)$ is a model of the $AF$ under complete semantics.
\end{proof}

Although there may not exist a one-to-one relationship between the set of models of an $AF$ under complete semantics and the set of models of $ec_2(AF)$ in $PL_3^L$, there exists a one-to-one relationship between the set of extensions of an $AF$ under complete semantics and the set of cores under the models of $ec_2(AF)$ in $PL_3^L$.

\section{Encoding $AF$s to $PL_{[0, 1]}$s}
From \cite{ref 29}, the normal encoding $ec_1$ combined with $PL_2$ constructs the stable semantics which is a ``$2$-valued complete semantics". Theorem \ref{thm5.3} establishes that the encoded semantics associated with $ec_1$ and $PL_3^L$ is equivalent to $3$-valued complete semantics. Therefore, $ec_1$ is a powerful encoding to construct the complete semantics w.r.t. different label sets. If we want to generalize to ``$[0,1]$-valued complete semantics" of $\mathcal{A}$, a natural thought is to encode $AF$s to $[0,1]$-valued logic systems by $ec_1$.

In this section, we encode $AF$s to fuzzy logic systems and explore the relationship between encoded semantics, real equational semantics and complete semantics by model relationships. We also obtain a specific encoded semantics associated with $ec_1$ and $PL_{[0,1]}^L$. Note that the label set of $AF$s is $[0,1]$.

\subsection{Encoding $AF$s to general $PL_{[0, 1]}$s}\label{sub5.1}
In this subsection, by encoding $AF$s to general $PL_{[0, 1]}$s, we obtain related encoded semantics, encoded equational systems and encoded equational functions. Then we explore the relationships between encoded semantics and the real equational semantics.
\subsubsection{Fuzzy normal encoded semantics}\label{5.1.1}
We first present the equational systems derived from fuzzy normal encoded semantics. For a $PL_{[0,1]}$ equipped with connectives $\neg$, $\wedge$ and $\rightarrow$, the three connectives are interpreted as operations negation $N$, t-norm $\ast$ and R-implication $I_\ast$, respectively.
  
\begin{theorem}\label{thm7}
 For an $AF=(A, R)$, an assignment $\|\cdot\|$ is a model of $ec_1(AF)$ in the $PL_{[0,1]}$ iff it is a solution of the equational system $Eq^{ec_1}$:  $\forall a\in A,$
 \begin{equation*}
 	\|a\| =\begin{cases}
 		1, &\nexists c((c,a)\in R),\\
 			N(\|b_1\|)\ast N(\|b_2\|)\ast\dots\ast N(\|b_k\|), &\text{otherwise},
 	\end{cases}
 \end{equation*}
 where $\{b_1, \dots , b_k \}$ is the set of all attackers of $a$. 
\end{theorem}
\begin{proof}
	An assignment $\|\cdot\|$ is a model of $ec_1(AF)$ in the $PL_{[0,1]}$\\
	 $\Longleftrightarrow$
	\begin{equation*}
		\|\bigwedge_{a\in A}(a\leftrightarrow\bigwedge_{(b_i,a)\in R}\neg{b_i})\|=1,
	\end{equation*}
	$\Longleftrightarrow$ for each $a \in A$, 
	\begin{equation*}
		\|(a\leftrightarrow\bigwedge_{(b_i,a)\in R}\neg{b_i})\|=1,
	\end{equation*}
	$\Longleftrightarrow$ for each $a \in A$, 
	\begin{equation*}
		\|(a\rightarrow\bigwedge_{(b_i,a)\in R}\neg{b_i})\wedge(\bigwedge_{(b_i,a)\in R}\neg{b_i})\rightarrow a\|=1,
	\end{equation*}
	$\Longleftrightarrow$ for each $a \in A$,
	\begin{equation*}
		\|a\rightarrow\bigwedge_{(b_i,a)\in R}\neg{b_i}\|=\|\bigwedge_{(b_i,a)\in R}\neg{b_i})\rightarrow a\|=1,
	\end{equation*}
	$\Longleftrightarrow$ for each $a \in A$, 
	\begin{equation*}
		I_\ast(\|a\|,\|\bigwedge_{(b_i,a)\in R}\neg{b_i}\|)=I_\ast(\|\bigwedge_{(b_i,a)\in R}\neg{b_i}\|, \|a\|)=1,
	\end{equation*}
	$\Longleftrightarrow$ for each $a \in A$, 
	\begin{equation*}
		\|a\|=\|\bigwedge_{(b_i,a)\in R}\neg{b_i}\|,
	\end{equation*}
	i.e.,
	 \begin{equation*}
		\|a\| =\begin{cases}
			1, &\nexists c((c,a)\in R),\\
			N(\|b_1\|)\ast N(\|b_2\|)\ast\dots\ast N(\|b_k\|), &\text{otherwise},
		\end{cases}
	\end{equation*}
	where $\{b_1, \dots , b_k \}$ is the set of all attackers of $a$.
\end{proof}

\begin{definition}\label{defn29}
The equational system $Eq^{ec_1}$ associated with a fuzzy normal encoded semantics is called the \textit{fuzzy normal encoded equational system} (abbreviated as the \textit{encoded equational system}).
The equational function associated with $Eq^{ec_1}$ is called the \textit{encoded equational function} and it is denoted by $h^{ec_1}: [0, 1]^k \to [0, 1]$, $h^{ec_1}(x_1, x_2, \dots, x_k )=N(x_1)\ast N(x_2)\ast\dots\ast N(x_k)$. 
\end{definition}
We also denote a fuzzy normal encoded semantics by $Eq^{ec_1}$.
Obviously, $h^{ec_1}$ may not be continuous since $N$ and $\ast$ may not be continuous. Therefore an encoded equational system $Eq^{ec_1}$ may not be a real equational system. In addition, we claim that a real equational system may not be an encoded equational system $Eq^{ec_1}$. On one hand, to clarify this claim we will introduce the property of decreasing monotonicity. On the other hand, to construct a proper semantics that aligns with the characteristics of argumentation, the decreasing monotonicity is essential but it is ignored in the definition of the real equational function. Intuitively, the value of an argument should decrease as the value of each of its attackers increases. Motivated by these considerations, we propose this property below.
\begin{definition}
	An equational function $h: [0, 1]^k \to [0, 1]$ satisfies the \textit{decreasing monotonicity} if:
	\begin{equation*}
		y \leqslant z \Longrightarrow h(x_1, \dots, x_{i-1}, y, x_{i+1}, \dots, x_k ) \geqslant h(x_1, \dots, x_{i-1}, z, x_{i+1}, \dots, x_k ). 
	\end{equation*} 
	In other words, $h$ is non-increasing in each variable.
\end{definition}

Then we have a theorem about the decreasing monotonicity below.  
\begin{theorem} \label{thm13}
	The encoded equational function $h^{ec_1}$ satisfies the decreasing monotonicity.
\end{theorem}
\begin{proof}
	Without loss of generality, let us assume that $x_1 \leqslant x_1'$. Then $N(x_1) \geqslant N(x_1')$. According to the associativity of t-norms, we have
	\begin{align*}
		h^{ec_1}(x_1, x_2, \dots, x_k)=&N(x_1)\ast N(x_2)\ast\dots\ast N(x_k)\\
		=&N(x_1)\ast( N(x_2)\ast\dots\ast N(x_k)).
	\end{align*}
	Because of the monotonicity of t-norms, we have 
	\begin{equation*}
		N(x_1)\ast( N(x_2)\ast\dots\ast N(x_k))\geqslant N(x_1')\ast( N(x_2)\ast\dots\ast N(x_k)).
	\end{equation*}
	Thus, we have
	\begin{align*}
		h^{ec_1}(x_1, x_2, \dots, x_k)&\geqslant N(x_1')\ast( N(x_2)\ast\dots\ast N(x_k))\\
		&= N(x_1')\ast N(x_2)\ast\dots\ast N(x_k)\\
		&=h^{ec_1}(x_1', x_2, \dots, x_k).
	\end{align*}
	Hence, $h^{ec_1}$ satisfies the decreasing monotonicity.
\end{proof}
\begin{remark}\label{rem3}
A real equational system $Eq^R$ may not satisfy the property of decreasing monotonicity, whereas an encoded equational system $Eq^{ec_1}$ necessarily exhibits this property from Theorem \ref{thm13}. Consequently, a real equational system may not be a fuzzy normal encoded equational system.
\end{remark}

Next we demonstrate that each encoded equational function $h^{ec_1}$ satisfies the boundary conditions and the symmetry.
\begin{theorem}\label{thm14}
	The encoded equational function $h^{ec_1}$ satisfies the boundary conditions and the symmetry:\\
	(a) $h^{ec_1}(0, \dots, 0) = 1$.\\
	(b) $h^{ec_1}(x_1, \dots , 1, \dots , x_k ) = 0$.\\
	(c) $h^{ec_1}(x_1, \dots, x_k) = h^{ec_1}(x_{\sigma(1)}, \dots, x_{\sigma(k)})$ for all $(x_1, \dots, x_k) \in [0,1]^k$ and all permutations $\sigma$ of $\{1, \dots, k\}$.
\end{theorem}
\begin{proof}
	Since $h^{ec_1}(x_1, x_2, \dots, x_k )=N(x_1)\ast N(x_2)\ast\dots\ast N(x_k)$, we have:\\
	(a) $h^{ec_1}(0, ... , 0) =N(0)\ast N(0)\ast\dots\ast N(0)=1\ast 1\ast\dots\ast 1= 1$,\\
	(b) $h^{ec_1}(x_1, \dots , 1, \dots , x_k ) =N(x_1)\ast\dots\ast N(1)\ast\dots\ast N(x_k)=N(x_1)\ast\dots\ast 0\ast\dots\ast N(x_k)= 0$,\\
	(c) We have 
	\begin{equation*}
		h^{ec_1}(x_1, \dots , x_k ) =N(x_1)\ast\dots\ast N(x_k)
	\end{equation*}
	 and 
		\begin{equation*}
		h^{ec_1}(x_{\sigma(1)}, \dots, x_{\sigma(k)} )=N(x_{\sigma(1)})\ast\dots\ast N(x_{\sigma(k)}).
	\end{equation*}
	 Because of the commutativity and the associativity of the t-norm, we have 
	 \begin{equation*}
	 	N(x_1)\ast\dots\ast N(x_k)=N(x_{\sigma(1)})\ast\dots\ast N(x_{\sigma(k)}). 
	 \end{equation*}
	 Thus 
	  \begin{equation*}
	 	h^{ec_1}(x_1, \dots , x_k ) = h^{ec_1}(x_{\sigma(1)}, \dots, x_{\sigma(k)}).
	 \end{equation*}	 
\end{proof}

\subsubsection{Continuous fuzzy normal encoded semantics}
In Subsubsection \ref{5.1.1}, we obtained the encoded equational system $Eq^{ec_1}$ and function $h^{ec_1}$, and also explored relationships between fuzzy normal encoded semantics and real equational semantics. In this part, we will study relationships between continuous (and strict) fuzzy normal encoded argumentation semantics and real equational semantics. Let $PL_{[0, 1]}^\star$ be a $PL_{[0,1]}$ equipped with a continuous negation $N^\star$, a continuous t-norm $\circledast$ and an $R$-implication $I_\circledast$. Let $PL_{[0, 1]}^{\star\star}$ be a $PL_{[0,1]}$ equipped with a strict negation $N^{\star\star}$, a continuous t-norm $\circledast$ and an $R$-implication $I_\circledast$.

\begin{theorem}\label{thm10}
	For an $AF=(A, R)$, an assignment $\|\cdot\|$ is a model of $ec_1(AF)$ in the $PL_{[0,1]}^\star$ iff it is a solution of the equational system $Eq_{ec_1}^\star$: $\forall a\in A$,
	\begin{equation*}
		\|a\| =\begin{cases}
			1, &\nexists c((c,a)\in R),\\
			N^\star(\|b_1\|)\circledast N^\star(\|b_2\|)\circledast\dots\circledast N^\star(\|b_k\|),  & \text{otherwise},
		\end{cases}
	\end{equation*}
	where $\{b_1, \dots , b_k \}$ is the set of all attackers of $a$. 
\end{theorem}
\begin{proof}
	Obviously from Theorem \ref{thm7}.
\end{proof}
\begin{definition}
	The equational system $Eq_{ec_1}^\star$ associated with a continuous fuzzy normal encoded semantics is called the \textit{continuous fuzzy normal encoded equational system} (abbreviated as the \textit{continuous encoded equational system}).
	The equational function associated with $Eq_{ec_1}^{\star}$ is called the \textit{continuous encoded equational function} and it is denoted by $h_{ec_1}^\star: [0, 1]^k \to [0, 1]$, $h_{ec_1}^\star(x_1, x_2, \dots, x_k )=N^\star(x_1)\circledast N^\star(x_2)\circledast\dots\circledast N^\star(x_k)$. 
\end{definition}
\begin{theorem}\label{thm11}
	For an $AF=(A, R)$, an assignment $\|\cdot\|$ is a model of $ec_1(AF)$ in the $PL_{[0,1]}^{\star\star}$ iff it is a solution of the equational system $Eq_{ec_1}^{\star\star}$: $\forall a\in A$,
	 \begin{equation*}
		\|a\| =\begin{cases}
			1, &\nexists c((c,a)\in R),\\
			N^{\star\star}(\|b_1\|)\circledast N^{\star\star}(\|b_2\|)\circledast\dots\circledast N^{\star\star}(\|b_k\|), &\text{otherwise},
		\end{cases}
	\end{equation*}
	where $\{b_1, \dots , b_k \}$ is the set of all attackers of $a$. 
\end{theorem}
\begin{proof}
	Obviously from Theorem \ref{thm7}.
\end{proof}
 \begin{definition}
 	The equational system $Eq_{ec_1}^{\star\star}$ associated with a strict fuzzy normal encoded semantics is called the \textit{strict fuzzy normal encoded equational system} (abbreviated as the \textit{strict encoded equational system}).
 	The equational function associated with $Eq_{ec_1}^{\star\star}$ is called the \textit{strict encoded equational function} and it is denoted by $h_{ec_1}^{\star\star}: [0, 1]^k \to [0, 1]$, $h_{ec_1}^{\star\star}(x_1, x_2, \dots, x_k )=N^{\star\star}(x_1)\circledast N^{\star\star}(x_2)\circledast\dots\circledast N^{\star\star}(x_k)$. 
 \end{definition}
 
\begin{theorem}\label{thm6.11}
	Each continuous encoded equational system $Eq_{ec_1}^\star$ is a Gabbay real equational system.
\end{theorem}

\begin{proof}
	We need to prove that each continuous encoded equational function $h_{ec_1}^{\star}$ is a Gabbay real equational function, where $h_{ec_1}^\star(x_1, x_2, \dots, x_k )=N^\star(x_1)\circledast N^\star(x_2)\circledast\dots\circledast N^\star(x_k)$.
	Since $N^\star$ and $\circledast$ are continuous, $h_{ec_1}^{\star}$ is continuous. From Theorem \ref{thm14}, $h_{ec_1}^{\star}$ satisfies the boundary conditions and the symmetry. Thus from Definition \ref{defn15}, each continuous encoded equational function $h_{ec_1}^\star$ is a Gabbay real equational function. Therefore, each continuous encoded equational system $Eq_{ec_1}^\star$ is a Gabbay real equational system.
\end{proof}

Since each strict encoded equational system is a continuous encoded equational system, from Theorem \ref{thm6.11}, we have the corollary below.
\begin{corollary}\label{cor.new}
	Each strict encoded equational system $Eq_{ec_1}^{\star\star}$ is a real equational system.
\end{corollary}
The relationships shown by Theorem \ref{thm6.11} and Corollary \ref{cor.new} can be expressed equivalently as follows.

\begin{corollary}\label{cor6.12}
	A model of $ec_1(AF)$ in a $PL_{[0, 1]}^\star$ is a real equational extension of the $AF$.
\end{corollary}
\begin{proof}
	Follows from Theorem\ref{thm10} and Theorem \ref{thm6.11}.
\end{proof}
\begin{corollary}
	A model of $ec_1(AF)$ in a $PL_{[0, 1]}^{\star\star}$ is a real equational extension of the $AF$.
\end{corollary}
\begin{proof}
	Follows from Theorem\ref{thm11} and Corollary \ref{cor.new}.
\end{proof}

By Theorem \ref{thm6.11} or Corollary \ref{cor6.12}, we can construct many specific real equational semantics via encoding $\mathcal{AF}$ to different $PL_{[0, 1]}^\star$s.

We state that the converse of Theorem \ref{thm6.11} is not true, i.e., a real equational system may not be a continuous encoded equational system $Eq_{ec_1}^{\star}$. It is obvious from Remark \ref{rem3}. However, we ask whether every real equational system whose real equational function satisfies decreasing monotonicity is a continuous encoded equational system. Below, we provide an example to address this question.
\begin{example}
	The equational system $Eq_{geometrical}^R$ \cite{ref 12} is associated with a real equational function $h_g(x_1, \dots , x_k ) =\frac{\prod_i(1-x_i)}{\prod_i(1-x_i)+\prod_ix_i}$. Using the knowledge of calculus, we know that $h_g$ satisfies the decreasing monotonicity. 
		
    Suppose $Eq_{geometrical}^R$ is derivable from the $ec_1(AF)$ model in some $PL_{[0, 1]}^\star$ with a continuous negation $N^\star$, a continuous t-norm $\circledast$ and an R-implication $I_\circledast$. 
	Consider $m \in A$ with a single attacker $n$. By $Eq_{geometrical}^R$,
	\begin{equation*}
	\|m\| = \frac{(1 - \|n\|)}{(1 - \|n\|) + \|n\|} = 1 - \|n\|.
	\end{equation*}
	Under the $ec_1(AF)$ model in $PL_{[0, 1]}^\star$, we have $\|m\| = N^\star(\|n\|)$. Thus,
	\begin{equation*}
		N^\star(\|n\|) = 1 - \|n\|.
	\end{equation*}
	Therefore, the negation is standard negation.
			
	Next, consider $a \in A$ with two attackers $b_1$ and $b_2$. By $Eq_{geometrical}^R$,
	\begin{equation*}
		\|a\| = \frac{(1 - \|b_1\|)(1 - \|b_2\|)}{(1 - \|b_1\|)(1 - \|b_2\|) + \|b_1\|\|b_2\|}. 
	\end{equation*}
   Via the $ec_1(AF)$ model in $PL_{[0, 1]}^\star$ (or by Theorem \ref{thm10}),
	\begin{equation*}
		\|a\| = N^\star(\|b_1\|) \circledast N^\star(\|b_2\|) = (1 - \|b_1\|) \circledast (1 - \|b_2\|).
	\end{equation*}
	Thus
	\begin{equation*}
		(1 - \|b_1\|) \circledast (1 - \|b_2\|) = \frac{(1 - \|b_1\|)(1 - \|b_2\|)}{(1 - \|b_1\|)(1 - \|b_2\|) + \|b_1\|\|b_2\|}.
	\end{equation*}
	Hence the t-norm $\circledast$ must satisfy
	\begin{equation*}
		x \circledast y = \frac{xy}{xy + (1-x)(1-y)}, \quad \forall x,y \in [0,1]. 
	\end{equation*}
			
	However, for $x \in (0,1)$:
	\begin{equation*}
		x \circledast 1 = \frac{x \cdot 1}{x \cdot 1 + (1-x)(1-1)} = \frac{x}{x + 0} = 1 \neq x.
	\end{equation*}
	This violates the $T_1$ in Definition \ref{defn12} ($x \circledast 1 = x$). Consequently, $\circledast$ is not a t-norm, contradicting the requirement that $PL_{[0, 1]}^\star$ operates under a t-norm $\circledast$.
			
	Thus, $Eq_{geometrical}^R$ cannot be derived from the $ec_1(AF)$ model in any $PL_{[0, 1]}^\star$.		
\end{example}

\begin{remark}
A real equational system, whose real equational function satisfies decreasing monotonicity, may not be a continuous encoded equational system.
\end{remark}

\subsection{Encoding $AF$s to specific $PL_{[0, 1]}$s}
In this subsection, we will encode $AF$s to three specific $PL_{[0, 1]}s$. The fuzzy normal encoded semantics associated with $PL_{[0, 1]}^G$ and $PL_{[0, 1]}^P$ are equivalent to real equational semantics $Eq_{\text{max}}^R$ and $Eq_{\text{inverse}}^R$, respectively. This paper proposes, for the first time, the fuzzy normal encoded semantics associated with $PL_{[0, 1]}^L$.

\subsubsection{Encoding $AF$s to $PL_{[0, 1]}^G$ and $PL_{[0, 1]}^P$}

First, we present the theorem of model equivalence between the real equational semantics $Eq_{\text{max}}^R$ and the fuzzy normal encoded semantics associated with $PL_{[0, 1]}^G$.

\begin{theorem}
An assignment is a model of an $AF$ under real equational semantics $Eq_{\text{max}}^R$ iff it is a model of $ec_1(AF)$ in $PL_{[0, 1]}^G$. 
\end{theorem} 
\begin{proof}
Let the $PL_{[0, 1]}$ in Theorem \ref{thm7} be $PL_{[0, 1]}^G$ with the standard negation, the G\"{o}del t-norm $\ast$ and the R-implication $I_\ast$. Then from Theorem \ref{thm7} for an $AF=(A, R)$:\\
an assignment $\|\cdot\|$ is a model of $ec_1(AF)$ in $PL_{[0, 1]}^G$ 
\\$\Longleftrightarrow$ it is a solution of the equational system: $\forall a\in A$,
\begin{equation*}
	\|a\| =\begin{cases}
		1, &\nexists c((c,a)\in R),\\
		(1-\|b_1\|)\ast (1-\|b_2\|)\ast\dots\ast (1-\|b_k\|), &\text{otherwise},
	\end{cases}
\end{equation*}
i.e.,
\begin{equation*}
	\|a\| =\begin{cases}
		1, &\nexists c((c,a)\in R),\\
		\min_{i=1}^k(1-\|b_i\|), &\text{otherwise},
	\end{cases}
\end{equation*}
i.e.,
\begin{equation*}
	\|a\| =\begin{cases}
		1, &\nexists c((c,a)\in R),\\
		1-\max_{i=1}^k\|b_i\|, &\text{otherwise},
	\end{cases}
\end{equation*}
where $\{b_1, \dots , b_k \}$ is the set of all attackers of $a$
\\$\Longleftrightarrow$ it is a model of the $AF$ under real equational semantics $Eq_{\text{max}}^R$.
\end{proof}

From this theorem, we have the corollary below. 
\begin{corollary}
The real equational semantics $Eq_{\text{max}}^R$ is translatable by $ec_1$ and $PL_{[0, 1]}^G$, i.e., it is equivalent to a fuzzy normal encoded semantics associated with $PL_{[0, 1]}^G$.
\end{corollary}

Note that, to emphasize the fuzzy normal encoded semantics associated with $PL_{[0, 1]}^G$, we also denote the real equational system $Eq_{\text{max}}^R$ as $Eq^G$. 

Then, we present the theorem of model equivalence between the real equational semantics $Eq_{\text{inverse}}^R$ and the fuzzy normal encoded semantics associated with $PL_{[0, 1]}^P$.

\begin{theorem}
	An assignment is a model of an $AF$ under real equational semantics $Eq_{\text{inverse}}^R$ iff it is a model of $ec_1(AF)$ in $PL_{[0, 1]}^P$. 
\end{theorem} 
\begin{proof}
	Let the $PL_{[0, 1]}$ in Theorem \ref{thm7} be $PL_{[0, 1]}^P$ with the standard negation, the Product t-norm $\ast$ and the R-implication $I_\ast$. Then from Theorem \ref{thm7} for an $AF=(A, R)$:\\
	an assignment $\|\cdot\|$ is a model of $ec_1(AF)$ in $PL_{[0, 1]}^P$ 
	\\$\Longleftrightarrow$ it is a solution of the equational system: $\forall a\in A$,
	\begin{equation*}
		\|a\| =\begin{cases}
			1, &\nexists c((c,a)\in R),\\
			(1-\|b_1\|)\ast (1-\|b_2\|)\ast\dots\ast (1-\|b_k\|), &\text{otherwise},
		\end{cases}
	\end{equation*}
	i.e.,
   \begin{equation*}
		\|a\| =\begin{cases}
			1, &\nexists c((c,a)\in R),\\
			\prod_{i=1}^{k}(1-\|b_i\|), &\text{otherwise},
		\end{cases}
	\end{equation*}
	where $\{b_1, \dots , b_k \}$ is the set of all attackers of $a$
	\\$\Longleftrightarrow$ it is a model of the $AF$ under real equational semantics $Eq_{\text{inverse}}^R$.
\end{proof}

From this theorem, we have the corollary below. 
\begin{corollary}
	The real equational semantics $Eq_{\text{inverse}}^R$ is translatable by $ec_1$ and $PL_{[0, 1]}^P$, i.e., it is equivalent to a fuzzy normal encoded semantics associated with $PL_{[0, 1]}^P$.
\end{corollary}
Note that to emphasize the fuzzy normal encoded semantics associated with $PL_{[0, 1]}^P$, we also denote the real equational system $Eq_{\text{inverse}}^R$ as $Eq^P$.

\subsubsection{Encoding $AF$s to $PL_{[0, 1]}^L$}
$PL_{[0, 1]}^L$ is one of the most important fuzzy propositional logic systems in \cite{ref 51}. It is necessary for us to propose a fuzzy normal encoded semantics associated with $PL_{[0, 1]}^L$. This encoded semantics exemplifies the power of the general encoding methodology.

We first present an equational system $Eq^L$ and then prove that the equational system $Eq^L$ is a fuzzy normal encoded equational system. For a given $AF=(A,R)$, we propose an equational system $Eq^L$: $\forall a\in A$
\begin{equation*}
	\|a\|=\begin{cases}
		1, &\nexists c((c,a)\in R),\\
		0, & \sum_{i=1}^{k} \| b_i \| \geqslant 1,\\
		1-\sum_{i=1}^{k} \| b_i \|, &\sum_{i=1}^{k} \| b_i \| < 1,
	\end{cases}
\end{equation*}
where $\|\cdot \|$ is a labelling function and $\{b_1, \dots , b_k \}$ is the set of all attackers of $a$. This equational system $Eq^L$ shows that $a$ is totally defeated, i.e., $\|a\|=0$, when the sum of values of all attackers of $a$ is not less than 1. 
Then we have the following theorem:
\begin{theorem}
For an $AF=(A, R)$, an assignment $\|\cdot\|$ is a model of $ec_1(AF)$ in $PL_{[0,1]}^L$ iff it is a solution of the equational system $Eq^L$.
\end{theorem} 
\begin{proof}
		First, for $\forall j\in\{1,2,\dots,n\}$, $x_j\in[0,1]$, we prove the equation
	\begin{equation}\label{eq6.1}
		x_1\ast x_2\ast\dots\ast x_n=\begin{cases}
			0 &\sum_{i=1}^{n} x_i  \leqslant n-1\\
			\sum_{i=1}^{n} x_i-n+1 &\sum_{i=1}^{n} x_i > n-1
		\end{cases}
	\end{equation}
	by induction, where $\ast$ is {\L}ukasiewicz t-norm. For brevity we denote $x_1\ast x_2\ast\dots\ast x_n$ by $T_n^\ast$. 
	
	Since $x\ast y= \max\{0, x +y-1\}$, Equation \ref{eq6.1} holds for $n=2$, i.e.,
	\begin{equation*}
		x_1\ast x_2=
		\begin{cases}
			0 & x_1+x_2 \leqslant 1 \\
			x_1+x_2-1 & x_1+x_2 > 1
		\end{cases}.
	\end{equation*}
	
	If we assume that for $n=k$ Equation \ref{eq6.1} holds, i.e.,
	\begin{equation*}\label{eq6.3}
		T_k^\ast=\begin{cases}
			0 &\sum_{i=1}^{k} x_i \leqslant k-1\\
			\sum_{i=1}^{k} x_i-k+1 &\sum_{i=1}^{k} x_i > k-1
		\end{cases},
	\end{equation*}
	then for $n=k+1$ we have
	\begin{equation}\label{eq6.4}
		T_{k+1}^\ast=T_k^\ast \ast x_{k+1}=\begin{cases}
			0 & T_k+x_{k+1}\leqslant 1 \\
			T_k+x_{k+1}-1 & T_k+x_{k+1}>1
		\end{cases}.
	\end{equation}
	Then we discuss three cases below.
	\begin{itemize}
		\item Case 1: $\sum_{i=1}^{k} x_i \leqslant k-1$\\
		$\Longleftrightarrow T_k=0$\\
		$\Longrightarrow T_{k}+x_{k+1}= x_{k+1}\leqslant1$\\
		$\Longrightarrow T_{k+1}=0.$ 
		\item Case 2: $k-1<\sum_{i=1}^{k} x_i\leqslant k-x_{k+1}$\\
		$\Longrightarrow T_k=\sum_{i=1}^{k} x_i-k+1$\\
		$\Longrightarrow T_k+x_{k+1}=\sum_{i=1}^{k} x_i+x_{k+1}-k+1\leqslant k-x_{k+1}+x_{k+1}-k+1=1$\\
		$\Longrightarrow T_{k+1}=0.$ 
		\item Case 3: $\sum_{i=1}^{k} x_i> k-x_{k+1}$\\
		$\Longrightarrow x_1+x_2+\dots+x_k> k-1$\\
		$\Longrightarrow T_k=\sum_{i=1}^{k} x_i-k+1$\\
		$\Longrightarrow T_k+x_{k+1}=\sum_{i=1}^{k} x_i+x_{k+1}-k+1>k-x_{k+1}+x_{k+1}-k+1=1.$
	\end{itemize}
	From the three cases and Equation \ref{eq6.4}, we have 
	\begin{equation*}
		T_{k+1}^\ast=
		\begin{cases}
			0 &\sum_{i=1}^{k+1} x_i \leqslant k\\
			\sum_{i=1}^{k+1} x_i-k &\sum_{i=1}^{k+1} x_i>k
		\end{cases},
	\end{equation*}
	i.e.,
	\begin{equation*}
		T_{k+1}^\ast=\begin{cases}
			0 &\sum_{i=1}^{k+1} x_i \leqslant (k+1)-1\\
			\sum_{i=1}^{k+1} x_i-(k+1)+1 &\sum_{i=1}^{k+1} x_i > (k+1)-1
		\end{cases}.
	\end{equation*}
	Then Equation \ref{eq6.1} holds for $n=k+1$. Thus Equation \ref{eq6.1} holds for $\forall n\in \mathbb{N}^+$, where $\mathbb{N}^+$ is the set of all positive integers.
	
	Second, we prove the final result.
	Let the $PL_{[0, 1]}$ in Theorem \ref{thm7} be $PL_{[0, 1]}^L$ with the standard negation, the {\L}ukasiewicz t-norm $\ast$ and the R-implication $I_\ast$. Then from Theorem \ref{thm7} for an $AF=(A, R)$:\\
	an assignment $\|\cdot\|$ is a model of $ec_1(AF)$ in $PL_{[0, 1]}^L$ 
	\\$\Longleftrightarrow$ it is a solution of the equational system: $\forall a\in A$,
	\begin{equation*}
		\|a\|=\begin{cases}
			1, &\nexists c((c,a)\in R),\\
			(1-\|b_1\|)\ast (1-\|b_2\|)\ast\dots\ast (1-\|b_k\|), & \text{otherwise},
		\end{cases}
	\end{equation*}
	i.e., 	(by Equation \ref{eq6.1})
	\begin{equation*}
		\|a\|=\begin{cases}
			1, &\nexists c((c,a)\in R),\\
			0, & \sum_{i=1}^{k} \| b_i \| \geqslant 1,\\
			1-\sum_{i=1}^{k} \| b_i \|, &\sum_{i=1}^{k} \| b_i \| < 1,
		\end{cases}
	\end{equation*}
	where $\{b_1, \dots , b_k \}$ is the set of all attackers of $a$
	\\$\Longleftrightarrow$ it is a solution of equational system $Eq^L$.
\end{proof}
From this theorem and Definition \ref{defn29}, we have the corollary below. 
\begin{corollary}
The equational system $Eq^L$ is a fuzzy normal encoded equational system associated with $PL_{[0, 1]}^L$.
\end{corollary}
From Theorem \ref{thm6.11} and the continuity of the fuzzy normal encoded equational system $Eq^L$, we have the corollary below. 
\begin{corollary}
	The fuzzy normal encoded equational system $Eq^L$ is a real equational system.
\end{corollary}

\subsection{Relationships between complete semantics and fuzzy normal encoded semantics}
In Subsection \ref{sub5.1}, we have explored the relationships between encoded semantics and real equational semantics. In this subsection, we examine the relationships between complete semantics and encoded semantics. We will consider encoded semantics associated with two important t-norm classes.

Suppose that the $PL_{[0,1]}$ in Theorem \ref{thm7} is equipped with a negation $N$, a $\frac{1}{2}$-idempotent t-norm $\circledcirc$ and an R-implication $I_\circledcirc$. We denote the corresponding encoded semantics by $Eq_\circledcirc^{ec_1}$:  $\forall a\in A,$
\begin{equation*}
	\|a\| =\begin{cases}
		1, &\nexists c((c,a)\in R),\\
		N(\|b_1\|)\circledcirc N(\|b_2\|)\circledcirc\dots\circledcirc N(\|b_k\|), &\text{otherwise},
	\end{cases}
\end{equation*}
where $\{b_1, \dots , b_k \}$ is the set of all attackers of $a$.  And  we denote an encoded equational function associated with $Eq_\circledcirc^{ec_1}$ by $h_\circledcirc^{ec_1}: [0, 1]^k \to [0, 1]$, $h_\circledcirc^{ec_1}(x_1, x_2, \dots, x_k )=N(x_1)\circledcirc N(x_2)\circledcirc\dots\circledcirc N(x_k)$.
\begin{theorem}\label{thm16}
	For a given $AF$ and an encoded semantics $Eq_\circledcirc^{ec_1}$, if an assignment $\|\cdot\|$ is a model of the $AF$ under the $Eq_\circledcirc^{ec_1}$, then $T_{\text{com}}(\|\cdot\|)$ is a model of the $AF$ under the complete semantics.
\end{theorem} 

\begin{proof}
	For a given $AF$, $\forall a\in A$, let $\{b_1, b_2, \dots, b_k\}$ be the set of all attackers of $a$. If an assignment $\|\cdot\|$ is a model of the $AF$ under the $Eq_\circledcirc^{ec_1}$, then denoting $T_{\text{com}}(\|\cdot\|)$ by $\|\cdot\|_{\text{com}}$, we have  
	\begin{equation*}
		\|a\|_{\text{com}}=
		\begin{cases}
			1 & \text{iff } \|a\|=1, \\
			0 & \text{iff } \exists b_i((b_i,a)\in R): \|b_i\|=1,\\
			\frac{1}{2} & \text{otherwise}.
		\end{cases}
	\end{equation*}
	Then we need to discuss three cases.
	\begin{itemize}
		\item Case 1, $\|a\|_{\text{com}}=0$.\\
		$\|a\|_{\text{com}}=0$
		\\$\Longleftrightarrow$  $\exists b_i((b_i,a)\in R): \|b_i\|=\|b_i\|_{\text{com}}=1$ 
		\\$\Longleftrightarrow$ $\|a\|_{\text{com}}=0$ satisfies complete semantics.

		In this case, we have:\\
		$\exists b_i((b_i,a)\in R): \|b_i\|=\|b_i\|_{\text{com}}=1$ 
		\\$\Longrightarrow$ $\|a\|=h_\circledcirc^{ec_1}(\|b_1\|, \|b_2\|,\dots,\|b_k\|) = N(\|b_1\|)\circledcirc N(\|b_2\|)\circledcirc\dots\circledcirc N(\|b_k\|)=0$.
		\\We also have:\\ 
		$\|a\|=0$ 
		\\$\Longrightarrow$ $N(\|b_1\|)\circledcirc N(\|b_2\|)\circledcirc\dots\circledcirc N(\|b_k\|)=0$
		\\$\Longrightarrow$ by the zero-divisor-free t-norm $\circledcirc$, 
		$\exists b_i((b_i,a)\in R): N(\|b_i\|)=0$ 
		\\$\Longrightarrow$ $\exists b_i((b_i,a)\in R): \|b_i\|=\|b_i\|_{\text{com}}=1$.
		\\Thus we have:\\
		$\|a\|=0$
		\\$\Longleftrightarrow$ $\exists b_i((b_i,a)\in R): \|b_i\|=\|b_i\|_{\text{com}}=1$
		\\$\Longleftrightarrow$ $\|a\|_{\text{com}}=0$.
		\item Case 2, $\|a\|_{\text{com}}=1$.\\
		$\|a\|_{\text{com}}=\|a\|=1$
		\\$\Longleftrightarrow$ $\nexists c((c,a)\in R)$ or $h_\circledcirc^{ec_1}(\|b_1\|, \|b_2\|,\dots,\|b_k\|) = N(\|b_1\|)\circledcirc N(\|b_2\|)\circledcirc\dots\circledcirc N(\|b_k\|)=1$ 
		\\$\Longleftrightarrow$ $\nexists c((c,a)\in R)$ or $N(\|b_1\|)= N(\|b_2\|)=\dots= N(\|b_k\|)=1$ 
		\\$\Longleftrightarrow$ $\nexists c((c,a)\in R)$ or $\|b_1\|= \|b_2\|=\dots= \|b_k\|=0$
		\\$\Longleftrightarrow$ $\nexists c((c,a)\in R)$ or $\|b_1\|_{\text{com}}= \|b_2\|_{\text{com}}=\dots= \|b_k\|_{\text{com}}=0$
		\\$\Longleftrightarrow$ $\|a\|_{\text{com}}=1$ satisfies complete semantics.
		\item Case 3, $\|a\|_{\text{com}}=\frac{1}{2}$.\\
		$\|a\|_{\text{com}}=\frac{1}{2}$
		\\$\Longleftrightarrow$ not the case $\|a\|_{\text{com}}=1$ or $\|a\|_{\text{com}}=0$
		\\$\Longleftrightarrow$ neither $\|a\|_{\text{com}}=1$ nor $\|a\|_{\text{com}}=0$ satisfies the complete semantics
		\\$\Longleftrightarrow$ $\|a\|_{\text{com}}=\frac{1}{2}$ satisfies complete semantics.
	\end{itemize}
	From the three cases, for a given model $\|\cdot\|$ of an $AF$ under $Eq_\circledcirc^{ec_1}$, $T_{\text{com}}(\|\cdot\|)$ is a model of the $AF$ under the complete semantics.
\end{proof}
This theorem shows that each model of an $AF$ under a encoded semantics $Eq_\circledcirc^{ec_1}$ can be turned into a model of the $AF$ under the complete semantics by the ternarization function. 

Next we explore the situation that a model of an $AF$ under complete semantics is a model of the $AF$ under a characterized encoded semantics. We give the characterized t-norm below.
\begin{definition}
	An \textit{$\frac{1}{2}$-idempotent t-norm} is a t-norm $\odot$ that satisfies $\frac{1}{2}\odot \frac{1}{2}=\frac{1}{2}$.
\end{definition}
Suppose that the $PL_{[0,1]}$ in Theorem \ref{thm7} is equipped with a standard negation, a $\frac{1}{2}$-idempotent t-norm $\odot$ and an R-implication $I_\odot$. We denote the corresponding encoded semantics by $Eq_\odot^{ec_1}$:  $\forall a\in A,$
\begin{equation}\label{odot}
	\|a\| =\begin{cases}
		1, &\nexists c((c,a)\in R),\\
		(1-\|b_1\|)\odot (1-\|b_2\|)\odot\dots\odot (1-\|b_k\|), &\text{otherwise},
	\end{cases}
\end{equation}
where $\{b_1, \dots , b_k \}$ is the set of all attackers of $a$.  And  we denote an encoded equational function associated with $Eq_\odot^{ec_1}$ by $h_\odot^{ec_1}: [0, 1]^k \to [0, 1]$, $h_\odot^{ec_1}(x_1, x_2, \dots, x_k )=(1-x_1)\odot (1-x_2)\odot\dots\odot (1-x_k)$.
\begin{theorem}\label{idem}
	For a given $AF$ and an encoded semantics $Eq_\odot^{ec_1}$, if an assignment $\|\cdot\|$ is a model of the $AF$ under the complete semantics then $\|\cdot\|$ is a model of the $AF$ under the $Eq_\odot^{ec_1}$.
\end{theorem} 

\begin{proof}
	For a given $AF$, $\forall a\in A$, let $\{b_1, b_2, \dots, b_k\}$ be the set of all attackers of $a$. If an assignment $\|\cdot\|$ is a model of the $AF$ under the complete semantics, then we have
	\begin{equation*}
		\|a\| = 
		\begin{cases} 
			1 & \text{iff } \nexists (b,a) \in R \text{ or } \forall (b_i,a) \in R: \|b_i\| = 0; \\ 
			0 & \text{iff } \exists (b_i,a) \in R: \|b_i\| = 1; \\ 
			\frac{1}{2} & \text{otherwise}.
		\end{cases}
	\end{equation*}
		Then we need to discuss three cases.
	\begin{itemize}
		\item Case 1, $\|a\|=0$ under complete semantics.\\
		$\|a\|=0$ under the complete semantics
		\\$\Longleftrightarrow$  $\exists b_i((b_i,a)\in R): \|b_i\|=1$ 
		\\$\Longrightarrow$ $h_\odot^{ec_1}(\|b_1\|, \|b_2\|,\dots,\|b_k\|) = (1-\|b_1\|)\odot (1-\|b_2\|)\odot\dots\odot (1-\|b_k\|)=0$
		\\$\Longleftrightarrow$ 
		$\|a\|=h_\odot^{ec_1}(\|b_1\|, \|b_2\|,\dots,\|b_k\|)$, i.e., Equation \ref{odot} holds under this case.
		\item Case 2, $\|a\|=1$ under the complete semantics.\\
		$\|a\|=1$ under the complete semantics
		\\$\Longleftrightarrow$ $\nexists c((c,a)\in R)$ or $\forall b_i ((b_i,a) \in R): \|b_i\| = 0$
		\\$\Longleftrightarrow$ $\nexists c((c,a)\in R)$ or $h_\odot^{ec_1}(\|b_1\|, \|b_2\|,\dots,\|b_k\|) = (1-\|b_1\|)\odot (1-\|b_2\|)\odot\dots\odot (1-\|b_k\|)=1$
		\\$\Longleftrightarrow$ 
		\begin{equation*}
			\|a\| =\begin{cases}
				1, &\nexists c((c,a)\in R),\\
				(1-\|b_1\|)\odot (1-\|b_2\|)\odot\dots\odot (1-\|b_k\|), &\text{otherwise},
			\end{cases}
		\end{equation*}
		i.e., Equation \ref{odot} holds under this case.
		\item Case 3, $\|a\|=\frac{1}{2}$ under the complete semantics.\\
		$\|a\|=\frac{1}{2}$ under the complete semantics
		\\$\Longleftrightarrow$  $\exists b_i ((b_i,a) \in R): \|b_i\| = \frac{1}{2}$ (i.e., $N(\|b_i\|) = \frac{1}{2}$) and $\nexists b_j ((b_j,a) \in R): \|b_j\| = 1$ (i.e., $N(\|b_j\|) = 0$)
			\\$\Longrightarrow$ $h_\odot^{ec_1}(\|b_1\|, \|b_2\|,\dots,\|b_k\|) = (1-\|b_1\|)\odot (1-\|b_2\|)\odot\dots\odot (1-\|b_k\|)=\frac{1}{2}$
		\\$\Longleftrightarrow$ $\|a\|=h_\odot^{ec_1}(\|b_1\|, \|b_2\|,\dots,\|b_k\|)$,
		i.e., Equation \ref{odot} holds under this case.
	\end{itemize}
	From the three cases, for a given model $\|\cdot\|$ of an $AF$ under the complete semantics, $\|\cdot\|$ is a model of the $AF$ under the $Eq_\odot^{ec_1}$.
\end{proof}
Suppose that the $PL_{[0,1]}$ in Theorem \ref{thm7} is equipped with a standard negation, a $\frac{1}{2}$-idempotent t-norm without zero divisors and an R-implication. We denote the associated encoded semantics by $Eq_\circleddash^{ec_1}$. Then from Theorem \ref{thm16} and Theorem \ref{idem}, we have a corollary below.
\begin{corollary}
	For a given $AF$, we have:
\begin{equation*}
		\Big\{\, \|\cdot\| \ \Big|\ \|\cdot\| \text{ is a model under complete semantics} \Big\} 
	= 
	\Big\{\, \|\cdot\|_{\text{com}} \ \Big|\ \|\cdot\| \text{ is a model under } Eq_{\circleddash}^{ec_1} \Big\}.
\end{equation*}
\end{corollary}

\section{Conclusion}
This paper advances the encoding methodology of $AF$s by systematically exploring their connections with diverse $PLS$s, aiming to strengthen the theoretical bridge between $AF$ theory and logical systems while providing novel approaches for constructing argumentation semantics.

To generalize the encoding methodology beyond classical $PL_2$, we encode $AF$s to $PL_3$s and $PL_{[0,1]}$s, employing two key encodings: the normal encoding ($ec_1$) and the regular encoding ($ec_2$). For 3-valued semantics, we established rigorous model relationships: stable semantics is equivalent to the encoded semantics associated with $PL_3^K$ and $ec_1$, while complete semantics corresponds to the encoded semantics of $PL_3^L$ and $ec_1$. For the regular encoding ($ec_2$), model relationships between complete semantics and the encoded semantics associated with $PL_3^K$ and $PL_3^L$ are further clarified via binarization and ternarization functions, expanding the applicability of $PL_3$s in argumentation analysis.

In the $[0,1]$-valued situations, we obtain the fuzzy encoded semantics by encoding $AF$s to $PL_{[0,1]}$s and explore their connections with Gabbay's real equational semantics. We show that continuous fuzzy normal encoded equational systems are special cases of real equational systems, while highlighting two key advantages of the proposed encoding methodology: first, encoded equational functions inherently satisfy decreasing monotonicity—a property not guaranteed by real equational functions; second, it leverages existing knowledge of $PLS$s to conveniently construct semantics that align with argumentation characteristics. We further demonstrate that Gabbay's real equational semantics $Eq_{\text{max}}^R$ and $Eq_{\text{inverse}}^R$ correspond exactly to the fuzzy encoded semantics associated with $PL_{[0,1]}^G$ and $PL_{[0,1]}^P$, respectively. Notably, we introduce a new fuzzy encoded semantics ($Eq^L$) associated with $PL_{[0,1]}^L$, which exhibits distinct features in evaluating argument acceptability by aggregating the strength of attackers.

Furthermore, we investigate the interplay between complete semantics and fuzzy encoded semantics, showing that ternarization of models under fuzzy encoded semantics associated with zero-divisor-free t-norms yields models of complete semantics, and that models of complete semantics can be embedded into specific fuzzy encoded semantics (equipped with $\frac{1}{2}$-idempotent t-norms). These results solidify the logical foundation of $AF$ semantics by grounding them in well-established logical systems.

Overall, this work generalizes the encoding methodology of $AF$s, establishes precise correspondences between argumentation semantics and encoded semantics in $PLS$s, and introduces a flexible framework for generating new argumentation semantics via logical encodings. Future research will focus on encoding higher-level $AF$s and bipolar $AF$s with support relations, further extending the general encoding methodology to these extended argumentation models.

\section*{Declarations}
\begin{itemize}
\item \textbf{Funding} No funds, grants, or other support was received.
\item \textbf{Competing interests} The authors have no competing interests to declare that are relevant to the content of this article.
\item \textbf{Ethics approval and consent to participate} Not applicable.
\item \textbf{Consent for publication} Not applicable.
\item \textbf{Data availability} Not applicable. 
\item \textbf{Materials availability} Not applicable. 
\item \textbf{Code availability} Not applicable. 
\end{itemize}

\begin{appendices}
	\section{Proofs of model equivalence for complete semantics}\label{secA1}
 We formally restate theorems of model equivalence related to complete semantics and then provide rigorous proofs. 
	
\begin{theorem}\label{thm4.3}
	For a given $AF=(A, R)$, if an assignment $\|\cdot\|$ is a model of $ec_2(AF)$ in $PL_2$, then $\|\cdot\|_{\text{com}}$ is a model of the $AF$ under complete semantics. 
\end{theorem}
		\begin{proof}
				If an assignment $\|\cdot\|$ is a model of $ec_2(AF)$ in $PL_2$, then we have:
			\begin{equation*}
				\|ec_2(AF)\|=\|\bigwedge_{a\in A}((a\rightarrow\bigwedge_{(b,a)\in R}\neg b)\wedge(a\leftrightarrow\bigwedge_{(b,a)\in R}\bigvee_{(c,b)\in R}c))\|=1
			\end{equation*}
			$\Longleftrightarrow \forall a\in A,$
			\begin{equation*}
				\|a\rightarrow\bigwedge_{(b,a)\in R}\neg b\|=\|a\leftrightarrow\bigwedge_{(b,a)\in R}\bigvee_{(c,b)\in R}c\|=1
			\end{equation*}
			$\Longleftrightarrow \forall a\in A,$
			\begin{equation}\label{ap1}
				\|a\rightarrow\bigwedge_{(b,a)\in R}\neg b\|=1
			\end{equation}
			and
			\begin{equation}\label{ap2}
				\|a\|=\|\bigwedge_{(b,a)\in R}\bigvee_{(c,b)\in R}c\|.
			\end{equation}
		$\Longleftrightarrow \forall a\in A,$
		\begin{equation*}
			\|a\|=
			\begin{cases}
				0, & \exists b_i ((b_i,a)\in R): [\nexists c ((c, b_i)\in R)\text{ or }\forall c_j ((c_j,b_i)\in R): \|c_j\|=0], \\
				1, & \nexists b_i ((b_i,a)\in R)\text{ or }\forall b_i((b_i,a)\in R): [\|b_i\|=0 \text{ and } \exists c_j ((c_j,b_i)\in R): \|c_j\|=1].
			\end{cases}
		\end{equation*}
		
		Next, we prove that $\|\cdot\|_{\text{com}}$ is a complete labelling according to Definition \ref{defn7}. We need to discuss three cases.
		
		\begin{itemize}
			\item Case 1, $\|a\|_{\text{com}}=1$.\\
			$\|a\|_{\text{com}}=1$
			\\$\Longleftrightarrow$ $\|a\|=1$ by the definition of $T_{\text{com}}$
			\\$\Longleftrightarrow$ $\nexists b_i ((b_i,a)\in R)$ or $\forall b_i((b_i,a)\in R): [\|b_i\|=0$ and $\exists c_j ((c_j,b_i)\in R): \|c_j\|=1]$
			\\$\Longleftrightarrow$ $\nexists b_i ((b_i,a)\in R)$ or $\forall b_i((b_i,a)\in R): [\|b_i\|_{\text{com}}=0$ and $\exists c_j ((c_j,b_i)\in R): \|c_j\|_{\text{com}}=1]$
			\\$\Longleftrightarrow$ $\nexists b_i ((b_i,a)\in R)$ or $\forall b_i((b_i,a)\in R): \|b_i\|_{\text{com}}=0$ by the equivalence between [$\|b_i\|_{\text{com}}=0$] and [$\exists c_j ((c_j,b_i)\in R): \|c_j\|_{\text{com}}=1$].
			
			Thus $\|a\|_{\text{com}}=1$ satisfies complete semantics under Case 1.
			\item 	Case 2, $\|a\|_{\text{com}}=0$.\\
			$\|a\|_{\text{com}}=0$
			\\$\Longleftrightarrow$ $\exists b_i ((b_i,a)\in R): \|b_i\|_{\text{com}}=\|b_i\|=1$, by the definition of $T_{\text{com}}$.
			
			Thus $\|a\|_{\text{com}}=0$ satisfies complete semantics under Case 2.		
			\item 	Case 3, $\|a\|_{\text{com}}=\frac{1}{2}$.\\
			$\|a\|_{\text{com}}=\frac{1}{2}$
			\\$\Longleftrightarrow$ $\|a\|_{\text{com}}\neq1$ and $\|a\|_{\text{com}}\neq0$
			\\$\Longleftrightarrow$ not Case 1 and not Case 2
			\\$\Longleftrightarrow$ not [$\nexists b_i ((b_i,a)\in R)$ or $\forall b_i((b_i,a)\in R): \|b_i\|_{\text{com}}=0$] and not [$\exists b_i ((b_i,a)\in R): \|b_i\|_{\text{com}}=\|b_i\|=1$]
			\\$\Longleftrightarrow$ $\exists b_i ((b_i,a)\in R): \|b_i\|_{\text{com}}\neq0$ and $\nexists b_i ((b_i,a)\in R): \|b_i\|_{\text{com}}=\|b_i\|=1$
		   \\$\Longleftrightarrow$ $\exists b_i ((b_i,a)\in R): \|b_i\|_{\text{com}}=\frac{1}{2}$ and $\nexists b_i ((b_i,a)\in R): \|b_i\|_{\text{com}}=\|b_i\|=1$. 
		   Thus $\|a\|_{\text{com}}=0$ satisfies complete semantics under Case 3.
		\end{itemize}
		Hence $\|\cdot\|_{\text{com}}$ is a complete labelling of the $AF$, i.e., $T_{\text{com}}(\|\cdot\|)$ is a model of the $AF$ under complete semantics.
	\end{proof}
	\begin{theorem}\label{thm4.2}
		For a given $AF=(A, R)$, if an assignment $\|\cdot\|$ is a model of the $AF$ under complete semantics, then $\|\cdot\|_2$ is a model of $ec_2(AF)$ in $PL_2$. 
	\end{theorem}
	
	\begin{proof}
		For the $AF=(A, R)$, if an assignment $\|\cdot\|$ is a model of the $AF$ under complete semantics, then	for all $a \in A$:
		\begin{equation*}
			\|a\| = 
			\begin{cases} 
				1 & \text{iff } \nexists (b,a) \in R \text{ or } \forall b((b,a) \in R): \|b\|= 0; \\ 
				0 & \text{iff } \exists (b,a) \in R: \|b\| = 1; \\ 
				\frac{1}{2} & \text{otherwise}.
			\end{cases}
		\end{equation*}
		and then for all $a \in A$:
		\begin{equation*}
			\|a\|_2 = 
			\begin{cases} 
				1 & \text{iff } \nexists (b,a) \in R \text{ or } \forall b((b,a) \in R): \|b\|= 0; \\ 
				0 & \text{otherwise}.
			\end{cases}
		\end{equation*}
	Since an assignment is a model of $ec_2(AF)$ in $PL_2$ iff it is a model of Equations \ref{ap1} and \ref{ap2} from the proof of Theorem \ref{thm4.3}, we only need to prove that $\|\cdot\|_2$ is a model of Equations \ref{ap1} and \ref{ap2}. 
	\begin{itemize}
		\item Case 1, $\|a\|_2 = 1$.
		\\$\|a\|_2 = \|a\|=1$
		\\$\Longleftrightarrow$ $\nexists (b,a) \in R$ or $\forall b_i( (b_i,a) \in R): \|b_i\|= 0$
		\\$\Longleftrightarrow$ $\nexists (b,a) \in R$ or $\forall b_i( (b_i,a) \in R)$: [$\|b_i\|= 0$ and $\exists c((c,b_i) \in R): \|c\| = 1$] by the equivalence between [$\|b_i\|= 0$] and [$\exists c((c,b_i) \in R): \|c\| = 1$] under complete labelling $\|\cdot\|$
 		\\$\Longrightarrow$ $\nexists (b,a) \in R$ or $\forall b_i( (b_i,a) \in R)$: [$\|b_i\|_2= 0$ and $\exists c((c,b_i) \in R): \|c\|_2 = 1$]
 		\\$\Longleftrightarrow$ Equations \ref{ap1} and \ref{ap2} hold.
		\item Case 2, $\|a\|_2 = 0$.
		\\$\|a\|_2 = 0$
		\\$\Longleftrightarrow$ $\|a\| = 0$ or $\|a\| = \frac{1}{2}$
		\\$\Longleftrightarrow$ $\|a\| \neq 1$
		\\$\Longleftrightarrow$ not [$\nexists b ((b,a) \in R)$ or $\forall b_i((b_i,a) \in R): \|b_i\|= 0$]
		\\$\Longleftrightarrow$ $\exists b_i((b_i,a) \in R): \|b_i\|\neq 0$
		\\$\Longleftrightarrow$ $\exists b_i((b_i,a) \in R)$$\nexists c((c,b_i)\in R )$: $\|c\|= 1$ 
		\\$\Longleftrightarrow$ $\exists b_i((b_i,a) \in R)$: [$\nexists c((c,b_i)\in R )$] or [$\forall c((c,b_i)\in R )$: $\|c\|\neq1$] 
		\\$\Longleftrightarrow$ $\exists b_i((b_i,a) \in R)$: [$\nexists c((c,b_i)\in R )$] or [$\forall c((c,b_i)\in R )$: $\|c\|_2= 0$] 
		\\$\Longleftrightarrow$ Equations \ref{ap1} and \ref{ap2} hold.
	\end{itemize}
	\end{proof}
	
\end{appendices}

\end{document}